\documentclass[11pt]{preprint}

\usepackage[utf8]{inputenc}
\usepackage{amsmath}
\usepackage{amsfonts}
\usepackage{amsthm}

\usepackage{geometry} 
\usepackage{graphicx} 

\usepackage{booktabs} 
\usepackage{array} 
\usepackage{paralist} 
\usepackage{verbatim} 
\usepackage{subfig} 

\usepackage{bm}

\usepackage[nottoc,notlof,notlot]{tocbibind} 
\usepackage[titles,subfigure]{tocloft} 

\usepackage{xcolor}
\usepackage{hyperref}
\hypersetup{
	colorlinks=true,
	linkcolor=cyan!50!blue,
	filecolor=blue,
	urlcolor=red,
	citecolor=green,
}

\newcommand{\R}{\mathbb{R}}
\newcommand{\Z}{\mathbb{Z}}

\newcommand{\I}{\mathbb{I}}

\newcommand{\Fcal}{\mathcal{F}}

\newcommand{\Bcal}{\mathcal{B}}
\newcommand{\Gcal}{\mathcal{G}}

\newcommand{\Hcal}{\mathcal{H}}
\newcommand{\Ical}{\mathcal{I}}

\newcommand{\Acal}{\mathcal{A}}
\newcommand{\Pcal}{\mathcal{P}}

\usepackage{lineno}
\usepackage{amssymb}
\usepackage{mathrsfs}

\DeclareMathOperator{\stab}{Stab}
\DeclareMathOperator{\coor}{coor}

\newcommand{\sG}{\mathscr{G}}
\newcommand{\sH}{\mathscr{H}}
\newcommand{\sS}{\mathscr{S}}
\newcommand{\sK}{\mathscr{K}}
\newcommand{\sT}{\mathscr{T}}

\newcommand{\fg}{\mathfrak{g}}
\newcommand{\fh}{\mathfrak{h}}
\newcommand{\ft}{\mathfrak{t}}

\newcommand{\fa}{\mathfrak{a}}
\newcommand{\fb}{\mathfrak{b}}
\newcommand{\fe}{\mathfrak{e}}
\newcommand{\fs}{\mathfrak{s}}

\DeclareMathOperator{\lip}{Lip}
\newcommand{\relu}{\mathrm{ReLU}}

\newcommand{\Hcalode}{\mathcal{H}_{\text{ode}}}

\newcommand{\ch}{\mathrm{CH}}
\newcommand{\chbar}{\overline{\mathrm{CH}}}

\usepackage{color}

\theoremstyle{plain}
\newtheorem{theorem}{Theorem}
\newtheorem{proposition}{Proposition}
\newtheorem{corollary}{Corollary}
\newtheorem{lemma}{Lemma}

\theoremstyle{definition}
\newtheorem{definition}{Definition}
\newtheorem{example}{Example}

\theoremstyle{remark}
\newtheorem{remark}{Remark}

\begin{document}

\title{
        Deep Neural Network Approximation of Invariant Functions through Dynamical Systems
}

\author{Qianxiao Li$^1$, Ting Lin$^2$ and Zuowei Shen$^3$}
\institute{National University of Singapore,\email{qianxiao@edu.nus.sg}
\and
Peking University, \email{lintingsms@pku.edu.cn}
\and
National University of Singapore, \email{matzuows@edu.nus.sg}
}

\maketitle

\begin{abstract}
    We study the approximation of functions which are invariant with respect to
    certain permutations of the input indices using flow maps of dynamical
    systems.  Such invariant functions includes the much studied
    translation-invariant ones involving image tasks, but also encompasses many
    permutation-invariant functions that finds emerging applications in science
    and engineering. We prove sufficient conditions for universal approximation
    of these functions by a controlled equivariant dynamical system,
    which can be viewed as a general abstraction of deep residual networks with symmetry constraints.
    These results not only imply the universal approximation for a variety
    of commonly employed neural network architectures for symmetric function approximation,
    but also guide the design of architectures with approximation guarantees
    for applications involving new symmetry requirements.
\end{abstract}

\tableofcontents

\section{Introduction}
\label{sec:intro}

Deep learning enabled significant progress in computer vision and natural
language processing, arguably due to its ability to exploit structures of the
data.  For example, convolution neural networks (CNN) targets translational
symmetry~\cite{azulay2018deep}, whereas recurrent neural networks (RNN) accounts
for causal structures and
time-homogeneity~\cite{li2021on,li_approximation_2022}.  Recently, we witness
increasing interest to apply machine learning to problems arising in science and
engineering~\cite{goh2017deep,butler2018machine,noe2020machine}.  Here, very
different structures present themselves in the underlying data.  For instance,
in modelling structure-property relationships involving atomic
systems~\cite{butler2018machine}, the input data often comes in the form of a
list of atoms with their respective property descriptions, and the goal is to
predict some macroscopic quantities depending on these input specifications.
Examples of such properties include energy, elasticity constants, etc.  In these
cases, there is symmetry with respect to permutations on the list of atoms, and
sometimes on a subset of the properties, e.g. 3D atomic coordinates, via a
choice of basis.  These transformations can be viewed as specific subgroups of
the symmetric group on the coordinates of the feature vectors, and such
transformation leaves the macroscopic property invariant.  In the computational
chemistry literature, current methods to tackle this rely on graphical
representations to induce
invariance~\cite{xieCrystalGraphConvolutional2018,chenGraphNetworksUniversal2019},
but such methods have limited applications if atomic properties are not spatial
coordinates, e.g. correlation functions~\cite{yeomans1992statistical}.  Hence,
to learn structure-property relationships in general, it is necessary to devise
methods to approximate functions that are invariant under the action of a
specified subgroup of the permutation group on the feature indices.

It is the purpose of this work to investigate the role of deep learning in approximating
functions that are invariant under the action of permutation subgroups on input indices.
This includes the CNN as a special case, with the subgroup being the group of translations.
However, to address new challenges arising in scientific applications, it is necessary to generalize the theory to accommodate other types of symmetries.
There is an interesting interaction between deep neural networks and symmetry that is worth noting.
On the one hand, enforcing a certain type of symmetry on any model hypothesis space necessarily restricts its approximation flexibility.
On the other hand, in function approximation applications the goal is to develop a hypothesis space that has the power
to approximate arbitrary functions.
This dilemma leads to considerable challenge in building hypothesis spaces to approximate symmetric function families,
especially when the symmetry group under consideration is complex~\cite{bogatskiy2020lorentz,finzi2020generalizing}.
In this regard, deep learning offers an attractive solution to this problem.
The distinguishing feature of deep learning, compared with traditional hypothesis spaces, is the presence of compositional structures.
A deep neural network represents a function in compositional form:
\begin{equation}
	\bm{F}(\bm x)
	= \bm g \circ \bm{F}_T \circ \dots \circ \bm{F}_1 (\bm x),
\end{equation}
where each $\bm F_i$ is a shallow neural network layer and $\bm g$ is the last layer
used to match output dimensions.
The crucial point is that each $\bm F_i$ and $\bm g$ can be made simple,
yet it can yield a complex $\bm F$ by simply increasing the number of layers $T$.
Suppose now that we want to build an $\bm F$ which is invariant under a transformation
$\fg$ on the input data.
Deep learning can accomplish this by simply choosing $\bm F_i$ to be equivariant,
i.e. $\bm F_i (\fg(\bm x)) = \fg (\bm F_i (\bm x))$
and $\bm g$ to be invariant, i.e. $\bm g(\fg (\bm x)) = \bm g(\bm x)$.
Check that this implies that $\bm F$ is $\fg$ invariant.
These symmetry restrictions may force each $\bm F_i$ and $\bm g$ to be simple,
but this now no longer necessarily limits approximation power of the deep neural network,
which builds complexity through increasing $T$.
In other words, by building complexity through composition, deep learning
may circumvent the contradicting requirements of symmetry and flexibility.


In this paper, we present some basic theoretical results on when approximation of arbitrary symmetric
functions can be achieved via compositional structures.
The approximation of functions through composition has been studied in a number of recent works.
For example, non-asymptotic optimal approximation rates for fully connected deep ReLU networks are obtained
in~\cite{shen2021optimal,lu2021deep,shen2019deep}.
In~\cite{shen2021deep,shen2021deepb,shen2021neural}, the approximation of fully connected deep network beyond the ReLU is discussed.
In particular, a family of simple activation functions is designed in~\cite{shen2021deepb},
whose corresponding neural networks can approximate an arbitrary continuous function with arbitrary accuracy by a fixed size network.
Here, our focus is on the interaction of composition and symmetry.
We note that while there are a number of results in the literature on universal approximation of
symmetric functions, they mostly focus on specific architectures and symmetry
groups~\cite{cohenwelling16,bogatskiy2020lorentz,finzi2020generalizing,yarotsky2018universal,
sannai2019universal,maron2019universality,ravanbakhsh2017equivariance,ravanbakhsh2020universal,NIPS2017_f22e4747,zweig2021functional}.
The results in this paper are of a different nature.
Our goal is to give general sufficient conditions for any architecture to achieve universal
approximation under any symmetry constraints induced by suitable permutation subgroups.
Thus, the results can be used to deduce universal approximation results for a variety of architectures
and symmetry groups in essentially the same way, including convolutional neural networks,
permutation-invariant networks, etc.
We illustrate this in Section~\ref{sec:application},
where we show that our results immediately imply universal approximation of residual versions of popular architectures used to approximate symmetric functions.
More importantly, these sufficient conditions can be used to guide the development of new architectures
to accommodate new symmetry structures that may arise in future applications.
We will give illustrations of this in the realm of property prediction for crystalline and amorphous materials~\cite{tian2022tbd}.

To study this problem mathematically, we will adopt the continuous idealization of
deep learning first explored in \cite{weinan2017proposal,li2017maximum,Haber2017}.
In particular, we derive general sufficient conditions for the approximation of
functions invariant to the aforementioned symmetries through composition - now in the form
of dynamics.
These results expands upon those in \cite{li2019deep} by incorporating symmetry considerations.

The paper is organized as follows.
In Section~\ref{sec:prelim}, we introduce notation and prior work on the analysis of approximation
by dynamical hypothesis spaces, which sets the stage for the analysis in this paper.
We then present the main approximation results and their applications in Section~\ref{sec:thm}.
The proofs of these results are presented in Section~\ref{sec:proofs}.
\section{Preliminaries}
\label{sec:prelim}

In this section, we introduce notations, definitions and present some
known previous results relevant to the main results in this paper.
Section~\ref{sec:prelim:jems} recalls main results in \cite{li2019deep}
on the approximation theory of flow maps without symmetry considerations.
Section~\ref{sec:prelim:group} introduces terminologies in group theory
used to describe discrete invariance and equivariance.
Based on these two concepts, in Section~\ref{sec:prelim:uap}
we provide some elementary results on universal approximation property
under invariance and equivariance,
which reveal challenges one encounters in formulating a general approximation result.

Throughout this paper, we adopt the following notations:
\begin{enumerate}
    \item We use boldface letters $\bm x,\bm y,\bm z$ for points in the Euclidean space $\R^n$.
    For scalars such as the component of these vectors, we use non-bold letters $x,y,z$.
    \item We use normal, non-bold letters like $f,g,h$ and $\alpha, \beta, \gamma$ for scalar-valued functions,
    shortened as \emph{functions}, and normal bold letters like $\bm{f}, \bm{g},
    \bm{h}$ and $\bm \alpha, \bm \beta, \bm \gamma$ for vector-valued functions,
    shortened as \emph{mappings}.
    \item We use script font letters
    $\sG, \sH, \sK$ to denote groups,
    and use fraktur font letters $\fa,\fb$ to denote elements in these groups.
    \item We use calligraphic font letters $\Acal, \Fcal, \Gcal$ to denote families of functions or mappings.
    \item
    Unless otherwise stated, we adopt periodic boundary conditions when specifying vector or tensor indices.
    That is, if $\bm x \in \R^n$, then $x_{n+1} := x_1$, $x_{-1} := x_{n}$, and so on.

\end{enumerate}

\subsection{Dynamical Hypothesis Spaces}
\label{sec:prelim:jems}

We first recall the problem formulation and main results in~\cite{li2019deep}
relevant to the present analysis.
The key problem investigated there is the approximation of functions through
dynamical evolution. In particular, associated with an ordinary differential equation (ODE)
\begin{equation}
    \dot{\bm z}(t) = \bm f_t (\bm z(t)),
    \qquad
    \bm z(0) = \bm x,
    \qquad
    t \in [0, T],
\end{equation}
is the mapping $\bm x \mapsto \bm z(T)$, which can be used to approximate functions
by choosing the vector field ${\bm f}_t$ from a family of functions $\mathcal{F}$.
We call $\Fcal$ a \emph{control family}, since it serves to control the dynamics of $\bm z$,
as in the study of optimal control of differential equations (see e.g.~\cite{evans1983introduction}).
In this process, the complexity of the resulting mapping depends on $\mathcal{F}$ and $T$.
The results in \cite{li2019deep} concerns the density of mappings of this type,
and we recall them below.
We first define the \emph{flow map} for time-homogenous continuous dynamical systems.
\begin{definition}[Flow map]
    Let $\bm f:\mathbb R^n \to \mathbb R^n$ be Lipschitz.
    We define the flow map associated with $\bm f$ at time horizon $T$
    as $\bm \phi(\bm f, T)(\bm x) = \bm z(T)$, where $\dot{\bm z}(t) = \bm f(\bm z (t))$ with initial data ${\bm z}(0) = \bm x$.
    It is well-known (see e.g.~\cite{Arnold1973Ordinary}) that the mapping
    $\bm \phi(\bm f,T)$ is Lipschitz for any real number $T$,
    and the inverse of $\bm \phi(\bm f,T)$ is $\bm \phi(-\bm f,T)$.
    In particular, the flow map is bi-Lipschitz.
\end{definition}
Based on the flow map, we can define the \emph{attainable set} for a given control family $\Fcal$,
which contains the compositions of flow maps generated by dynamics driven by
vector fields chosen from $\Fcal$.
\begin{definition}[Attainable set]
    For a given control family $\Fcal$ of Lipschitz mappings,
    we define its attainable set
    \begin{equation}
        \Acal_{\Fcal} :=
        \{
        \bm{\phi}(\bm f_1,\tau_1) \circ \cdots \circ \bm \phi(\bm f_k,\tau_k)
        :
        \bm f_1,\cdots, \bm f_k \in \Fcal,
        \tau_j \geq 0,
        k \geq 1
        \}.
    \end{equation}
\end{definition}

The attainable set builds complexity through compositions of flow maps, and can
be used to approximate mappings in $\R^n$.  However, often in applications we
aim to approximate relationships whose range is not $\R^n$.  An example is
scalar regression problems, where we aim to construct approximations of
functions from $\R^n$ to $\R$.  Thus, to achieve approximation we require an
additional composition with a terminal family of functions to fix the range.
This gives rise to the following dynamical hypothesis space on which we study
approximation properties.

\begin{definition}[Dynamical hypothesis space]
    Given a control family $\Fcal$ and a terminal family $\Gcal$ of functions
    from $\R^n$ to $\R$, both assumed Lipschitz,
    the dynamical hypothesis space $\Hcalode$ is defined as
    \begin{equation}
        \Hcalode = \Hcalode(\Fcal, \Gcal) := \{ \bm g \circ \bm \varphi : \bm g \in \Gcal, \bm\varphi \in \Acal_{\Fcal}\}.
    \end{equation}
\end{definition}

The key approximation problem seeks conditions on $\Fcal$ and $\Gcal$
that induce the density of $\Hcalode$ in appropriate function spaces.
This is also known as the universal approximation property
in the machine learning literature.
To establish such results in appropriate generality,
the concept of \textit{well functions} was introduced in \cite{li2019deep}
in order to provide sufficient conditions to achieve
universal approximation.
Here, we recall its definition.

\begin{definition}[Well function]
    \label{defn:well-function}
    We say a Lipschitz function $h: \mathbb{R}^n \to \mathbb{R}$ is a well function
    if there exists a bounded open convex set $\Omega \subset \R^n$ such that
    \begin{equation}
        \{ \bm x \in \R^n : h(\bm x) = 0\} = \overline{\Omega}
    \end{equation}
    Moreover, we say that a vector valued function $\bm h : \R^n \to \R^{n'}$
    is a well function if each of its component $h_i:\R^n \to \R$ is a well function in the sense above.
    Specifically, a Lipschitz function $h : \R \to \R$ is a one-dimensional well function if $\{x: h(x) = 0\}$ is a non-degenerate closed interval.
\end{definition}

We now state the main density result in \cite{li2019deep} concerning the dynamical hypothesis space for $n \geq 2$.
In the following, for any collection $\Fcal$ of functions on $\R^d$, we denote by $\ch(\Fcal)$ its convex hull and
$\chbar(\Fcal)$ its closure in the topology of compact convergence.

\begin{theorem}[Main result in \cite{li2019deep}]
    \label{thm:jems}
    Suppose $n \ge 2$.
    Let $\bm {F}: \mathbb R^n \to \mathbb R^m$, $m\geq 1$, be continuous.
    If the control family $\mathcal F$ and the terminal family $\mathcal G$
    are both Lipschitz and satisfy
    \begin{enumerate}
        \item For any compact set $K \subset \mathbb{R}^n$,
        there exists $\bm g \in \mathcal{G}$ such that $\bm F(K) \subset \bm g(\mathbb R^n)$.
        \item $\Fcal$ is restricted affine invariant.
        That is, $\bm f \in \Fcal$ implies $Df(A\cdot + \bm b) \in \Fcal$,
        where $\bm b \in \R^n$ is any vector,
        and $D,A$ are any $n \times n$ diagonal matrices
        such that the entries of $D$ are $\pm 1$ or $0$,
        and the entries of $A$ are smaller than or equal to 1 in absolute value.
        \item $\chbar(\Fcal)$ contains a well function.
    \end{enumerate}
    Then for any $p \in [1,\infty)$, compact $K \subset \mathbb{R}^n$ and $\varepsilon>0$,
    there exists $\hat{\bm F} \in \mathcal{H}_{ode}$ such that
    $\| \bm F - \hat{\bm F}\|_{L^p(K)} \le \varepsilon.$
\end{theorem}

The straightforward but important deep learning application of Theorem~\ref{thm:jems}
is when $\Fcal = \Fcal_{\sigma} := \{V\sigma(W\cdot+b)\}$,
where $\sigma$ is a nonlinear activation function, such as $\relu$ or the
Sigmoid function.  This then corresponds to a universal approximation theorem of
deep residual neural networks\footnote{While Theorem~\ref{thm:jems} is stated in the
continuous setting, basic numerically analysis shows that the result
carries over to discretized architectures, see Sec.~\ref{sec:thm:discrete}}
through composing fixed-width residual layers,
which are building blocks of residual neural networks~\cite{He2016}.  However,
we note that Theorem~\ref{thm:jems} makes
no explicit reference to neural networks and can be viewed
as a general result on approximation of functions by the flow maps of dynamical systems.

The main purpose of this paper is to derive similar results under symmetry constraints,
where we aim to establish an analogue density result for group invariant dynamical hypothesis spaces.
One may wonder if the following simple argument suffices:
if we additionally require $\Fcal$ to be equivariant and $\Gcal$ to be invariant, then $\Hcalode$ is indeed invariant and thus if all conditions are Theorem~\ref{thm:jems} are satisfied, we deduce the universal approximation property.
It turns out that this argument is vacuous, since the restricted affine invariance condition cannot be satisfied for general control families that are equivariant with respect to a non-trivial subgroup of the permutation group.
For example, full permutation equivariance will force $D,A$ to be multiples of the identity matrix.
Hence, to make headway we will have to suitably relax the affine invariance condition.
This will be the primary challenge in establishing the main results of this paper,
and further highlights the competition between restrictions induced by symmetry
and flexibility required for universal approximation.

\subsection{Group Theory Notation}
\label{sec:prelim:group}

In this section, we fix some terminologies involving basic group theory.
The readers are expected to be familiar with groups and subgroups.
By $\sG \le \sH$ we mean $\sG$ is a subgroup of $\sH$.

\paragraph{Permutation Groups.} Given a finite set $S$, a permutation group on $S$ consists of some permutations on $S$ which form a group under the composition.
Without loss of generality, we can identify $S$ with $\{n\} := \{1,\cdots,n\}$,
and all groups considered here will be permutation groups.
For fixed $S$, the group of all permutations is called the symmetric group, denoted as $\sS$.
The identity element is denoted as $\fe$.
We denote by $(i~j)$ the transposition
element in $\sS$ that exchanges $i,j$, keeping the others fixed.

\paragraph{Group Actions on Indices and Vectors.}
Given a permutation $\fs$ on $\{n\}$, it is natural to describe how $\fs$ acts on $\{n\}$.
We use $i \mapsto \fs(i)$ to denote this mapping. For a vector $\bm x = [x_1,\cdots,x_n]
\in \R^n$, we define the action on $\R^n$ by
$
    \fs(\bm x) = [x_{\fs(1)},\cdots, x_{\fs(n)}]
$
and for a point set $A$, define $\fs(A) := \{ \fs(\bm x): \bm x \in A \}$. All the transformation in $\R^n$ considered in this paper are of this type.

\paragraph{Transitivity.}
We call a permutation group $\sG$ on $\{n\}$ transitive if for each $i,j \in \{n\}$,
there exists a permutation $\fg \in \sG$ such that $\fg(i) = j$.

\paragraph{Stabilizer.}
Given $\sG \le \sS$, define the \emph{stabilizer} of element $i$ as $\stab_i(\sG) = \{ \fg \in \sG : \fg(i) = i\}$.
A basic fact is that if $\fa(i) = j$, then $\fa\stab_i(\sG)\fa^{-1} = \stab_j(\sG)$.
Thus, for a transitive group $\sG$, all stabilizers are conjugate to each other.

\paragraph{Cross Section.}
Given $\fg \in \sS$, we define the \emph{cross section}
\begin{equation}
    Q_{\fg} := \{\bm{x} \in \R^n: x_{\fg^{-1}(1)} > x_{\fg^{-1}(2)} \cdots > x_{\fg^{-1}(n)}\}.
\end{equation}
If we write $Q := Q_{\fe}$,
the cross section of the identity element keeping all indices fixed,
then we have $Q_{\fg} = \fg(Q)$.

\paragraph{Transversal.}
Given $\sG \le \sS$, and denote by $k:= |\sS|/|\sG|$. We say a collection
$A:= \{\fa_1,\fa_2,\cdots,\fa_k\} \subset \sS$
is a (right) \emph{transversal} with respect to $\sG$ if
\begin{itemize}
    \item $\fa_i \fa_j^{-1} \not \in \sG$ for $i\neq j$,
    \item For any $\fb \in \sS$, there exists $\fa_i \in A$ such that $\fa_i \fb^{-1} \in \sG$.
\end{itemize}
We define the \emph{cross section} with respect to $A$ as
\begin{equation}
    \label{eq:cross-section-A}
    Q_{A} := \bigcup_{i = 1}^{k} Q_{\fa_i}.
\end{equation}

\paragraph{Invariant functions and Equivariant Mappings.}

We now give precise definitions of invariant and equivariant mappings, together with related concepts.
Let $\sG$ be a permutation subgroup.
We say $f: \R^n \rightarrow \R$ is a $\sG$ \emph{invariant function} if for all $\fg \in \sG$,
\begin{equation}
    f(\fg(\bm x)) = f(\bm x).
\end{equation}
We say $\bm \varphi : \R^n \rightarrow \R^n$ is a
$\sG$ \emph{equivariant mapping} if for all $\fg \in \sG$,
\begin{equation}
    \bm \varphi(\fg(\bm x)) = \fg(\bm \varphi(\bm x)).
\end{equation}
We introduce some examples of invariant functions and equivariant mappings to close this part.
These motivate us to build approximation frameworks under such symmetry considerations,
and serve as important examples for the application of our theory to deep learning subsequently.
\begin{example}[Translation]
    We first recall the definition of translation in one and multiple dimensions.
    For the one dimensional case, we define $\sT = \{\ft^1,\ft^2,\cdots,\ft^n\}$, where
    $\ft(i) = i+1$ is the shift operator in one dimension.
    Recall that the periodic condition is assumed.
    For the multidimensional case,
    we define the translation group as $\sT_{d_1}\times \cdots \times \sT_{d_N}$,
    for $n = d_1\cdots d_N$, and
    \begin{equation}
        \ft_k([i_1,\cdots,i_N]) = [i_1,\cdots,\ft(i_k),\cdots,i_N].
    \end{equation}
    Here, the ambient Euclidean space is identified as
    $\R^{n} := \R^{d_1} \times \cdots \times \R^{d_N}$,
    and the corresponding multi-index $[i_1, \cdots, i_N]$ is a sequence of length $N$ and $i_s \in \{d_s\}$.

    One of the most commonly used $\sT$ equivariant mapping is constructed
    by the convolution operation.
    Given $\bm w = [w_1,\cdots, w_p]$, the convolution operation\footnote{
        As is customary in the deep learning literature,
        this definition of convolution is in fact the \emph{cross correlation} and differs from the
        classical convolution in the order of indices for the filters.
        Note that such conventions do not affect any approximation results,
        so we choose to stick to the usual deep learning convention.
    }
    in one dimension is
    \begin{equation}\label{eq:conv-1d}
        \bm w \ast \bm x
        :=
        \left[
            \sum_{k=1}^{p} x_{k} w_{k},
            \sum_{k=1}^{p} x_{k+1} w_{k},
            \cdots,
            \sum_{k=1}^{p} x_{n-1+k} w_{k}
        \right].
    \end{equation}
    We also write down the definition the general $N$-dimensional case using the
    multi-index $\bm{i} = (i_1,\cdots,i_N)$, $1 \leq i_k \leq d_k$.
    The convolution operation in $N$ dimensions is
    \begin{equation}\label{eq:conv-nd}
        [\bm w \ast \bm x]_{\bm i} =
        \sum_{\bm k}
        x_{\bm k + \bm i}
        w_{\bm k},
    \end{equation}
    where the summation is over $k_j = 1,\dots,d_j$
    and $\bm k + \bm i := [k_1+i_1 - 1,\cdots,k_N+i_N - 1]$.
    The $N=2$ case is most relevant to image applications.
\end{example}

\begin{example}[Full permutation symmetry]
    In this example, we consider the symmetric group $\sS$.
    Some common $\sS$ invariant functions include
    $x_1+x_2+\cdots+x_n$ and $x_1x_2\cdots x_n$.
    $\sS$ equivariant mapping may be built from these, such as
    \begin{equation}
        [x_1, x_2, \cdots, x_n]
        \mapsto
        [x_1+x_2+\cdots+x_n,\cdots,x_1+x_2+\cdots + x_n],
    \end{equation}
    or
    \begin{equation}
        [x_1, x_2, \cdots, x_n]
        \mapsto
        [x_1+x_2x_3\cdots x_n, x_2+x_1x_3\cdots x_n, \cdots x_n + x_1x_2\cdots x_{n-1}].
    \end{equation}
    Functions respecting full permutation symmetry are often used for the study of
    physical systems whose attributes are set-like features with no order structures,
    e.g. a list of constituent atoms in a crystal lattice.
    These are also called set functions, which features in a variety of recent
    studies~\cite{zweig2021functional,NIPS2017_f22e4747,qi2017pointnet}.
\end{example}

The results of this paper apply not only to the aforementioned symmetries, but
also other types of partial permutation subgroups that naturally arise in
scientific applications.
We will discuss this in greater detail in Section~\ref{sec:application}.

\subsection{Universal Approximation Property under Invariance}
\label{sec:prelim:uap}

The approximation setting studied in this paper concerns the universal approximation
property (UAP) under symmetry induced by a permutation group $\sG \le \sS$.
The following definition makes this precise.

\begin{definition}[$\sG$ UAP]
    \label{def:guap}
    Let $\Hcal$ be a family of $\sG$ invariant functions from
    $\R^n\to \R^m$. $\Hcal$ is said to possess the $\sG$ \textbf{universal
    approximation property} ($\sG$ UAP) in $L^p$ sense if for any $\sG$
    invariant continuous function $\bm F : \R^n \to \R^m$, compact set $K\subset \R^n$
    and $\varepsilon>0$, there exists $\hat{\bm F} \in \Hcal$ such that
    \begin{equation}
        \|\bm F- \hat{\bm F}\|_{L^p(K)} \le \varepsilon.
    \end{equation}
    Similarly, let $\Acal$ be a family of $\sG$ equivariant mappings from $\R^n \to \R^n$.
    $\Acal$ is said to possess $\sG$ UAP if for any $\sG$ equivariant continuous mapping $\bm \varphi : \R^n \to \R^n$,
    compact set $K\subset \R^n$ and $\varepsilon>0$, there exists $\bm{\hat \varphi} \in \Acal$ such that
    \begin{equation}
        \|\bm \varphi - \bm{\hat \varphi}\|_{L^p(K)} \le \varepsilon.
    \end{equation}
\end{definition}

Before discussing the main result of this paper,
we highlight that symmetry constraints naturally limit approximation capabilities.
More concretely, if a function $\bm F$ (resp. mapping $\bm \varphi$) can be approximated by $\sG$
invariant functions (resp. equivariant mapping), then $\bm F$ (resp. $\bm \varphi$) itself is $\sG$ invariant (resp. equivariant), see the following proposition.

\begin{proposition}[Closure property of invariant functions]
    \label{prop:negative-invar}
    \label{prop:negative-equivar}
    Suppose $\bm F \in C(\R^n, \R^m)$ and for any compact $K\subset \R^n$, tolerance
    $\varepsilon > 0$, there exists a $\sG$ invariant function $\hat{\bm F}\in
    C(\mathbb R^n)$ such that $\|\bm F - \hat{\bm F}\|_{L^p(K)} < \varepsilon $.  Then
    $\bm F$ is $\sG$ invariant.

    Moreover, suppose $\bm{\varphi} \in C(\R^n; \R^n)$ and for any compact
    $K \subset \R^n$, tolerance $\varepsilon > 0$, there exists a $\sG$ equivariant
    mapping $\bm{\hat{\varphi}}$ such that $\|\bm{\varphi} -
    \bm{\hat{\varphi}}\|_{L^p(K)} < \varepsilon$.  Then $\bm{\varphi}$ is $\sG$
    equivariant.
\end{proposition}

\begin{proof}
    Given $\fg \in \sG$,
    compact set $K$ and $\varepsilon>0$,
    we choose $\hat{\bm F}$ as stated.
    Since $\hat{\bm F}$ is $\sG$ invariant, we have
    \begin{equation}
        \begin{split}
            \|\bm F(\bm{x}) - \bm F(\fg(\bm{x}))\|_{L^p(K)} &\le \|\bm F(\bm{x}) - \hat{\bm F}(\bm x)\|_{L^p(K)} + \|\hat{\bm F}(\bm{x}) - \hat{\bm F}(\fg(\bm{x}))\|_{L^p(K)}  \\ &~~~~~~+ \|\hat{\bm F}(\fg(\bm{x})) - \bm F(\fg(\bm x))\|_{L^p(K)}\\
            &\le 2\varepsilon.
        \end{split}
    \end{equation}
    This holds because $\fg$ is measure-preserving,
    so $\|\hat{\bm F}\circ \fg - \bm F\circ \fg\|_{L^p(K)} = \|\hat{\bm F}- \bm F\|_{L^p(K)}$.
    Note that both $\bm F$ and $\bm F \circ \fg$ are continuous, and $\varepsilon>0$ is arbitrary, yielding that $\bm F = \bm F\circ\fg$. Since this holds for all $\fg \in \sG$, we conclude that $\bm F$ must be $\sG$ invariant.

    The proof of the second part is similar.
    Given $\fg \in \sG$, and compact set $K$ and $\varepsilon>0$,
    we choose $\hat{\varphi}$ as stated.
    Then we have
    \begin{equation}
        \begin{split}
            \|\fg \circ \bm \varphi - \bm \varphi \circ \fg\|_{L^p(K)} \le ~& \|\fg \circ \bm \varphi - \fg \circ  \hat{\bm \varphi}\|_{L^p(K)} + \|\fg \circ  \hat{\bm \varphi} -  \hat{\bm \varphi} \circ \fg\|_{L^p(K)} + \| \hat{\bm \varphi} \circ \fg - \bm \varphi \circ \fg\|_{L^p(K)} \\ \le ~& 2\varepsilon,
        \end{split}
    \end{equation}
    since $\hat{\bm \varphi}$ is $\sG$ equivariant.
    Note that both $\fg \circ \bm \varphi$ and $\bm \varphi \circ \fg$ are continuous,
    and $\varepsilon$ can be arbitrarily chosen, yielding that
    $\bm \varphi\circ \fg = \fg \circ \bm \varphi$. Hence $\bm \varphi$ must be $\sG$ equivariant.
\end{proof}

An immediate consequence is that if $\sG \le \sH$, and our hypothesis space consists of $\sH$ invariant functions, then functions that are $\sG$ invariant but not $\sH$ invariant cannot be approximated to arbitrary precision.
Although we only consider finite permutation groups here, similar arguments yield that
the same limitation arises for continuous groups, raising a significant
problem in designing equivariant/invariant neural networks, e.g. under SO or SE symmetry \cite{weiler20183d}.
Generally, this suggests that the construction of invariant and equivariant architectures will be much more challenging
if the structure of $\sG$ is complex, and using composition gives a convenient way to build complexity while preserving
symmetry.
We will illustrate this point in Section~\ref{sec:application} through applications.

\section{Main Results}
\label{sec:thm}

In this section, we present our main result on the universal approximation of $\sG$ invariant functions
via dynamics driven by equivariant control families.
Concretely, let us fix a transitive subgroup $\sG \leq \sS$ and
consider a target function $\bm F : \R^n \rightarrow \R^m$ that is $\sG$ invariant.
For brevity, we will hereafter use invariant (resp. equivariant) to mean $\sG$ invariant (resp. $\sG$ equivariant).
We begin with definitions that are required to state our main result.

\subsection{Universal Approximation of Invariant Functions by Flows}

\subsubsection{The coor Operator.}
First, we introduce a way to associate with each equivariant control family $\Fcal$, which
consists of mappings $\R^n \rightarrow \R^n$ with a scalar-valued function on $\R^n$ that is
invariant with respect to a stabilizer.
This will allow us to introduce suitable extensions of the concept of well functions
that induce universal approximation under symmetry constraints.
The association rest on the transitive property of $\sG$ (See Section~\ref{sec:prelim:group}), as shown by the following result.

\begin{proposition}
    \label{prop:representation}
    Let $\sG$ be transitive and denote by $\sH := \stab_1(\sG)$.
    Then, for each $\sG$ equivariant mapping $\bm f : \R^n \rightarrow \R^n$, $\bm f = [f_1,\cdots,f_n]$,
    $f_1$ is $\sH$ invariant.
    Conversely, let $f_1 : \R^n \rightarrow \R$ be $\sH$ invariant and
    $\ft_k \in \sG$, $k = 1,2,\cdots,n$ satisfy $\ft_k(1) = k$.
    Then, we can construct a $\sG$ equivariant mapping $\bm{f}$ from $f_1$ as follows
    \begin{equation}
        \label{eq:stab-equi}
        \bm{f}(\bm x) = [f_1(\ft_1(\bm x)), f_1(\ft_2(\bm x)), \cdots, f_1(\ft_n(\bm x))].
    \end{equation}
\end{proposition}

\begin{remark}
    The above proposition is presented with respect to the first coordinate, e.g. $f_1$ and $\stab_1$.
    This choice is arbitrary and identical results holds true for any $j \in \{n\}$,
    since it follows from Section~\ref{sec:prelim:group}
    that all stabilizer are conjugate under the transitivity assumption.
\end{remark}

\begin{proof}
    Choose $\fh \in \sH$, then since $\fh \circ \bm f = \bm f \circ \fh$,
    we have $f_1(\fh(\bm x)) = f_1(\bm x)$. Thus $f_1$ is $\sH$ invariant.
    Conversely, given an $\sH$ invariant function $f_1$,
    we can verify that the construction \eqref{eq:stab-equi} is $\sG$ equivariant.
    Given $\fg \in \sG$, suppose $\fg(1) = k$, then we have $\fg = \ft_k \circ \fh$ for some $\fh \in \sH$.

    Now it suffices to check $\bm f \circ \ft_k = \ft_k \circ \bm f$.
    Consider the $j$-th coordinate,
    the left hand side becomes $f_1(\ft_j^{-1} \ft_k(\bm x))$,
    while the right hand side becomes $f_1(\ft_{l}^{-1}\bm x)$,
    where $l = \ft_k^{-1}j$.
    We only need to check that $\ft_j^{-1} \ft_k \ft_{l}$ is $\sH$ invariant.
    Direct calculation yields
    \begin{equation}
        \ft_j^{-1} \ft_k \ft_{l}(1) = \ft_j^{-1} \ft_k l = \ft_j^{-1} \ft_k \ft_k^{-1} j = \ft_j^{-1}j = 1.
    \end{equation}
    Therefore, $\ft_j^{-1} \ft_k \ft_{l}$ is  $\sH$ invariant.
    Combining with $\bm f \circ \fh= \fh \circ \bm f$ gives
    $\bm f \circ \fg = \fg \circ \bm f$.
\end{proof}

A consequence of Proposition~\ref{prop:representation} is that any $\sG$ equivariant mapping can be represented
by a scalar-valued function that is $\sH$ invariant with respect to the stabilizer $\sH:=\stab_1(\sG)$.
Based on this observation, we introduce the $\coor$ operator as follows.
\begin{definition}[$\coor$ operator]\label{def:coor}
    Let $\sG$ be a transitive subgroup and $\Fcal$ be a collection of $\sG$ equivariant mappings on $\R^n$.
    We define the $\coor$ operator
    \begin{equation}
        \coor(\Fcal) := \bigg\{
            f_1 : [f_1 \circ \ft_1,\cdots,f_1\circ \ft_n] \in \Fcal,
            \ft_i \in \sG,
            \ft_i(1) = i
        \bigg\}.
    \end{equation}
\end{definition}
The following result shows that $\coor(\Fcal)$ characterizes $\Fcal$,
and hence we may consider conditions on $\coor(\Fcal)$ directly in our approximation results, and the conditions on $\Fcal$ will be subsequently determined.
Its proof is by direct construction from Proposition~\ref{prop:representation}.
\begin{proposition}[Reconstruction from $\coor$ operator]
    Assume that $\sG$ is transitive, and $\sH = \stab_1(\sG)$.
    Given a function family $\Gcal$ containing $\sH$ invariant functions from $\R^n$ to $\R$,
    there exists a unique $\Fcal$, containing some $\sG$ equivariant mappings
    from $\R^n$ to $\R^n$ such that $\coor(\Fcal) = \Gcal$.
\end{proposition}

\subsubsection{Symmetric Invariant Well Functions.}
Let us now introduce a class of {\it symmetric invariant well functions}, which plays a central role in our
analyses of function approximation using composition or dynamics.
We have recalled the definition of {\it well function} introduced in \cite{li2019deep} to prove approximation results
without symmetry constraints in Section~\ref{sec:prelim:jems}, Definition~\ref{defn:well-function}.
Here, we will modify the notion of well functions to incorporate symmetry considerations.

\begin{definition}[Symmetric invariant well functions] \label{def:gen_well_fn}
Let $\tau : \R^n \rightarrow \R$ be a $n$-dimensional Lipschitz function.
We call it a symmetric invariant well function if the following conditions hold:
\begin{enumerate}
\item $\tau$ is $\sS$ invariant.
\item There exists a finite interval $\I  \subset \R$ such that if
$\bm x \in \I^n$
then $\tau(\bm x) = 0$.
\item Given $i \in \{n\}$ and $b > 0$,
there exists $a > 0$ such that $\tau([x_1,\cdots,x_n])\neq 0 $
for all $|x_i| > a$ and $|x_j| < b$ for $j \neq i$.
\end{enumerate}
\end{definition}

\begin{remark}
    \label{rmk:symmetrc-invariant-well-function}
    It is easy to verify that $[x_1,\cdots,x_n] \mapsto h(x_1+x_2+\cdots+x_n)$
    and $[x_1,\cdots,x_n] \mapsto h(x_1)+h(x_2)+\cdots+h(x_n)$ are both symmetric invariant well functions,
    provided that $h$ is a well function in the sense of Definition~\ref{defn:well-function}.

    We note that this definition is close to,
    but more general than just requiring a well function
    (in the sense of Definition~\ref{defn:well-function}) to be $\sS$ invariant.
    An $\sS$ invariant well function is also a symmetric invariant well function,
    but the converse does not hold in general.
    For example, for any one dimensional well function $h$, we consider $h(x_1+\cdots+x_n)$ again.
    This is a symmetric invariant well function in the sense of Definition~\ref{def:gen_well_fn},
    but it is not a well function since its zero set is unbounded.
    The current broader definition allows for easier application
    of our results to practical architectures.
\end{remark}

    In order to have universal approximation with symmetry constraints,
    it is necessary to rule out some degenerate cases where
    $\Fcal$ does not have the desired control granularity.
    In particular, since we require $\Fcal$ to be equivariant,
    it is possible that two initial conditions that share
    some coordinate values may continue to have identical coordinate
    values under the flow constructed from $\Fcal$,
    thereby limiting approximation capabilities.
    We make this precise by defining the following perturbation
    property that we will subsequently require $\Fcal$ to fulfill.
    We first define a special type of transformation
    function, which we call \textit{coordinate zooming functions}, that we use throughout
    this work to induce a non-linear and
    $\sS$ equivariant change of coordinates.

\begin{definition}[Coordinate zooming function]
    Let $u\in C(\R)$ be an increasing function.
    We define
    \begin{equation}
        u^{\otimes}(\bm{x}) := [u(x_1),\cdots,u(x_n)],
    \end{equation}
    which we call the \emph{coordinate zooming function} with respect to $u$.
    Clearly $u^{\otimes}$ is in $C(\R^n;\R^n)$ and is $\sS$ equivariant.
\end{definition}

\begin{definition}
We say a point $\bm x$ is in \emph{general position},
if none of its coordinates are the same.
\end{definition}
    \begin{definition}[Perturbation property]
        \label{def:pert_prop}
        Define the \emph{similarity} of two points
        \begin{equation}
            \label{eq:similarity}
            \overline{s}(\bm x, \bm y) = |\{ i:x_i = y_j\}|
        \end{equation}
        if at least one of them is in general position.
        Otherwise, the similarity is defined as 0.
        We say that $\Fcal$ satisfies the \emph{perturbation property} if
        for any $\bm x, \bm y$ with $\overline{s}(\bm x, \bm y) \neq 0$,
        there exists $\bm f \in \Fcal$ and a coordinate zooming function $u^{\otimes}$,
        such that $x_i = y_i$ for some $i$,
        but $[\bm f(u^{\otimes} (\bm x))]_i \neq [\bm f(u^{\otimes} (\bm y))]_i$.
    \end{definition}

\subsubsection{Sufficient Conditions for Approximation: Full Permutation Case}
With these definitions in mind, we now state our main universal approximation results,
starting with the full permutation group $\sS$.

\begin{theorem}
    \label{thm:main}
    Let $n\geq 2$ and $\bm F:\R^n \rightarrow \R^m$ be continuous and $\sS$ invariant.
    Suppose that the control family $\Fcal$ is $\sS$ equivariant
    and the terminal family $\Gcal$ is $\sS$ invariant,
    satisfying the following conditions
    \begin{enumerate}\addtocounter{enumi}{0}
        \item For any compact $K\subset \R^n$,
        there exists a Lipschitz $\bm g \in \Gcal$ such that
        $\bm F(K) \subset \bm g(\R^n)$.
        \item
        $\chbar(\Fcal)$ satisfies the perturbation property (Defintion~\ref{def:pert_prop}).
        In addition, $\Fcal$ is scaling and translation invariant along $\bm 1$,
        \textit{i.e.},
        $\bm f \in \Fcal$ implies $a\bm f(b\bm x + c\bm 1) \in \Fcal$
        for any $a,b,c \in \R$.
        \item $\chbar(\coor(\mathcal{F}))$ contains both a symmetric invariant well function $\tau$,
        and a function with the form $\bm x \mapsto h(x_1)$
        such that $h$ is a one-dimensional well function.
    \end{enumerate}
    Then, $\Hcalode(\Fcal,\Gcal)$ satisfies the $\sS$ UAP.
    That is, for any $p \in [1,\infty)$, compact $K\subset \R^n$ and $\varepsilon > 0$,
    there exists an $\sS$ invariant $\hat{\bm F} \in \Hcalode(\Fcal, \Gcal)$ such that
    \begin{equation}
        \| \bm{F} - \hat{\bm F} \|_{L^p(K)} \le \varepsilon.
    \end{equation}
\end{theorem}

Theorem~\ref{thm:main} is a parallel to Theorem~\ref{thm:jems} for the approximation of
$\sS$ invariant functions via equivariant dynamical systems.
In the statement of the result,
condition 1 mirrors that of  Theorem~\ref{thm:jems}.
Conditions 2 and 3 here replace those
in Theorem~\ref{thm:jems} regarding affine invariance and well functions by
with symmetry conditions captured by the $\coor$ operator.
As discussed earlier,
condition 2 is necessarily different from the restricted affine invariance
assumed in Theorem~\ref{thm:jems}, since the requirement that
$\Fcal$ be $\sS$ equivariant precludes the satisfaction
of restricted affine invariance in the general sense.
Thus, we require a weaker affine invariance assumption here,
together with the additional requirement on the perturbation property.

\begin{remark}
We give some examples to explain why we require at least
two types of well functions ---
a symmetric invariant one and a one-dimensional ($\bm x \mapsto h(x_1)$) one ---
in condition 3 of Theorem~\ref{thm:main}.
Consider
\begin{equation}\label{eq:counter_exp_gamma}
    \dot{\bm x} = \gamma(\bm x) \mathbf{1},
\end{equation}
where $\bm \gamma$ is an arbitrary scalar function.
In this case, $\Fcal := \{ \gamma(\bm x) \mathbf{1}:\gamma \in C(\R^n)\}$.
Check that $\Fcal$ satisfies Condition 2 in Theorem~\ref{thm:main}
and $\chbar(\coor(\Fcal))$
contains a symmetric well function, but not $\bm x \mapsto h(x_1)$.
Observe that any flow map $\bm \varphi \in \Acal_{\Fcal}$ will satisfy
$[\bm \varphi(\bm x)]_i - [\bm \varphi(\bm x)]_j = x_i - x_j$.
As a result, $\Acal_{\Fcal}$ cannot approximate every $\sS$ equivariant
function. Further, consider the terminal family as a single scalar function
$g(\bm x) = \max x_i - \min x_i$, then for all $\bm F \in \Hcalode(\Fcal, \{\bm g\})$,
it holds that $F(\bm x) = g(\bm x)$. This shows the approximation
property of invariant functions is therefore limited.
Theorem~\ref{thm:jems} assumes a stronger affine invariance condition
than Theorem~\ref{thm:main}.
Thus, while $\gamma$ in~\eqref{eq:counter_exp_gamma} can be taken as a well function,
the associated control family $\Fcal$ does not satisfy the restricted affine invariance property in Theorem~\ref{thm:jems}.

On the other hand, if we only consider coordinate-wise well functions,
then the following dynamics driven by a coordinate-zooming function
\begin{equation}
    \dot{\bm x} = u^{\otimes}(\bm x),
\end{equation}
can be constructed to satisfy all other conditions of Theorem~\ref{thm:main}.
In this case, $\Acal_{\Fcal}$ only consists of coordinate-wise mappings,
and hence cannot approximate every $\sS$ equivariant function for $n\geq 2$.
For instance, if we choose the terminal family as a single function $g(\bm x) = \max x_i$, then we can conclude that $\Hcalode(\Fcal,\{\bm g\})$ cannot approximate all $\sS$ invariant functions.
\end{remark}

\subsubsection{Sufficient Condition for Approximation: The General Case}
Next, we discuss the generalization of this result to transitive subgroups $\sG \leq \sS$.
While an $\sS$ equivariant control family is automatically $\sG$ equivariant,
we need additional requirements on $\Fcal$ to ensure that it is not too symmetric to
lose the ability to resolve mappings that are $\sG$ invariant but not $\sS$ invariant.
That is, we need to ensure that $\Fcal$ has enough resolution to balance with the structure of $\sG$.
The following example demonstrates a negative result when $\Fcal$ does not possess enough resolution.

\begin{example}
    \label{exa:failure}
We consider the translation group $\sG = \sT$,
the singleton terminal family $\Gcal = \{ g(\bm x) = x_1+ x_2 + x_3 \}$,
and the control family
\begin{equation}
    \mathcal{F} = \{[x_1,x_2,x_3] \mapsto [ah(bs+cx_1 +d) ,ah(bs+cx_2 +d) , ah(bs+cx_2 +d)]:a,b,c,d\in \mathbb{R}\}.
\end{equation}
It is easy to see Condition 2 in Theorem~\ref{thm:main} holds for $\Fcal$.
To verify Condition 3, we notice that, in this case
\begin{equation}
    \coor(\Fcal) = \{ a h(b(x_1+x_2+x_3) + c(x_1)+ d)\}.
\end{equation}
Clearly, $h(x_1) + h(-1-x_1) \in \chbar(\coor(\Fcal))$ is a one-dimensional well function.
Further, it follows from Remark~\ref{rmk:symmetrc-invariant-well-function} that
$h(x_1+x_2+x_3) $ is a symmetric invariant well function.

Therefore, Conditions 1-3 are satisfied for this control system.
However, this control family can only produce $\sS$ equivariant mappings.
By Proposition~\ref{prop:negative-invar}, it suffices to construct a function which is $\sT$ invariant but not $\sS$ invariant.
We provide a concrete example:
\begin{equation}
    \label{eq:approx_counter_example}
    f(\bm x) = (x_1 - x_2)(x_2 - x_3)(x_3 - x_1),
\end{equation}
which is $\sT$ equivariant, but not $\sS$ equivariant.
This function cannot be approximated by $\Hcalode(\Fcal,\Gcal)$ defined above.
\end{example}
In this specific case, the failure to approximate~\eqref{eq:approx_counter_example} in Example~\ref{exa:failure} can be explained as follows. Recall the definition of cross section in \eqref{eq:cross-section-A}.
Observe that any flow map generated from $\Fcal$ maps $Q_{\fa}$ into itself.
In other words, the flow does not have enough resolution to steer points across different cross sections.
However, this motion is necessary if we want to achieve approximation when $\sG \neq \sS$.

Motivated by this observation, we need a condition on $\Fcal$ with respect to
the symmetry group $\Gcal$, such that $\Acal_{\Fcal}$ can map points $Q_{\fa}$
to those in another cross section $Q_{\fb}$, for any $\fa,\fb$ belong to
different $\sG$ orbits.
This is a necessary condition for approximation of $\sG$ equivariant mappings.
We now discuss sufficient conditions to ensure this.

\begin{definition}[Direct Connectivity]
\label{defn:direct_connectivity}
We say two permutation $\fa$ and $\fb$ are directly $\Fcal$ connected if $\fa = (i~j)\fb$ for some $i, j \in \{n\}$, and there exists a $\bm z \in \partial Q_{\fa} \cap \partial Q_{\fb}$, and $\bm f \in \Fcal$, such that
    $
        [\bm f(\bm z)]_i \neq [\bm f(\bm z)]_j.
    $
We say two permutation $\fa$ and $\fb$ are $\Fcal$ connected if there exists
$\fg_0 = a, \fg_1, \cdots, \fg_s = \fb$, where $\fg_i$ and $\fg_{i-1}$ are directly $\Fcal$ connected.
\end{definition}




    \begin{definition}[Resolving a Group]
    \label{defn:resolvence}
    We say $\Fcal$ resolves $\sG$ if $\Fcal$ satisfies the perturbation property
    (Definition~\ref{def:pert_prop}), and moreover, it is $\sG$ transversally transitive:
    i.e. there exists a transversal $A$ such any two distinct elements $\fa,\fb \in A$ are $\Fcal$ connected.
    \end{definition}



Now, we are ready to state our main approximation result that applies to any transitive $\sG \leq \sS$.

\begin{theorem}
    \label{thm:main2}
    Let $\sG \leq \sS$ be transitive,
    $n\geq 2$ and $\bm F:\R^n \rightarrow \R^m$ be continuous and $\sG$ invariant.
    Suppose that the control family $\Fcal$ is $\sG$ equivariant
    and resolves $\sG$, while the terminal family $\Gcal$ is $\sG$ invariant,
    satisfying the following conditions
    \begin{enumerate}\addtocounter{enumi}{0}
        \item For any compact $K\subset \R^n$,
        there exists a Lipschitz $\bm g \in \Gcal$ such that $\bm F(K) \subset \bm g(\R^n)$.
        \item $\Fcal$ is scaling and translation invariant along $\bm 1$, \textit{i.e.},
        $\bm f \in \Fcal$ implies $a\bm f(b\bm x + c\bm 1) \in \Fcal$ for any $a,b,c \in \R$.
        \item $\chbar(\coor(\mathcal{\bm F}))$ contains both a symmetric invariant well function $\tau$,
        and a function $\bm x\mapsto h(x_1)$, where $h$ is a one-dimensional well function.
    \end{enumerate}
    Then, $\Hcalode(\Fcal,\Gcal)$ satisfies the $\sG$ UAP.
    That is, for any $p \in [1,\infty)$, compact $K\subset \R^n$ and $\varepsilon > 0$,
    there exists a $\sG$ invariant $\hat{\bm F} \in \Hcalode(\Fcal,\Gcal)$ such that
    \begin{equation}
        \| \bm F - \hat{\bm F} \|_{L^p(K)} \le \varepsilon.
    \end{equation}
\end{theorem}

Observe that Theorem~\ref{thm:main2} generalizes Theorem~\ref{thm:main},
since any $\sS$ equivariant $\Fcal$ necessarily possesses $\sS$ transversal transitivity.
Furthermore, we emphasize that theorem~\ref{thm:main2} is not a strict generalization of Theorem~\ref{thm:jems}.
While we may take $\sG = \{ \fe \}$ to remove the symmetry constraints,
the trivial subgroup $\{ \fe \}$ is not transitive, hence the result here do not apply.
In the proof of this theorem, we in fact prove a more general result regarding
the approximation of invariant functions by composing families of equivariant functions (Proposition~\ref{thm:level-set}).
This result applies broadly to compositional hypothesis spaces and is not limited to ODE-type dynamical systems.
In fact, it subsumes both Theorem~\ref{thm:jems} and Theorem~\ref{thm:main},
and may be viewed as an abstract sufficient condition for universal approximation through composition.
We defer the detailed discussion of this result to Section~\ref{sec:thm:composition}.

In the application of these results,
one only need to check whether a network architecture of interest
 satisfies the above four conditions.
Conditions 1 and 2 are easy to check.
The other conditions require additional effort,
but we will show in Section~\ref{sec:application}
that both known popular and useful novel
architecture types are included in our results,
and their verification are presented in Section~\ref{sec:proofs:verification}.

\subsection{Temporal Discretization}
\label{sec:thm:discrete}
Our main result applies to continuous-time dynamics,
which corresponds to a continuous-layer idealization of practical deep neural networks
\cite{weinan2017proposal,Haber2017,li2017maximum}.
It is thus natural to consider the implications of our results for the discrete case in applications.
In this subsection, we discuss how temporal discretization of continuous-time flows can inherit approximation properties.

In the following, we consider for simplicity only single step integrators for discretizing the continuous flow.
This includes but is not limited to the important example of forward Euler discretization,
which corresponds to the ResNet family of architectures~\cite{He2016}.
The numerical integrator is abstracted as follows: for a given $\bm f \in \Fcal$,
for each time interval of size $t$
we define the numerical scheme as a mapping
$\hat{\bm \phi}(\bm f, t)(\cdot) : \R^{n} \to \R^n$,
serving as an approximation of the
flow map of the continuous dynamics $\bm\phi(\bm f, t)(\cdot)$.
For a time horizon $T$,
the discrete flow with respect to the partition
$\Delta: 0 = t_0 < t_1 <\cdots <t_s = T$
is given as
\begin{equation}
    \hat{\bm \phi}(\bm f, \Delta)
    :=
    \hat{\bm \phi}(\bm f, t_{s} - t_{s-1})
    \circ
    \hat{\bm \phi}(\bm f, t_{s-1} - t_{s-2})
    \cdots
    \hat{\bm \phi}(\bm f, t_{1} - t_{0})
\end{equation}

\begin{definition}[Convergence of Numerical Scheme]
\label{defn:convergence}
We say the numerical scheme is convergent if
\begin{enumerate}
    \item $\hat{\bm \phi}(\bm f, t) \to \operatorname{Id}$ uniformly in any compact set $K \subset \R^n$, as $t \to 0$, where $\operatorname{Id}$ is the identity mapping.
    \item For any $T$,
    \begin{equation}
        \hat{\bm \phi}(\bm f, \Delta) \to
        \bm\phi(\bm f, T)
    \end{equation}
    uniformly in any compact set $K \subset \R^n$,
    as $|\Delta| := \max (t_i - t_{i-1}) \to 0$.
    \item There exists a continuous function
    $g:\R^n \times \R \rightarrow \R$
    such that
    $|\hat{\bm \phi}(\bm f,\Delta)(\bm x)| \le g(\bm x, T)$.

\end{enumerate}
\end{definition}

As an example, we consider the forward Euler scheme, where
\begin{equation}
\label{eq:Euler}
    \hat{\bm \phi}(\bm f, t)(\bm x) := \bm x + t \bm f(\bm x).
\end{equation}
Standard numerical analysis shows that the forward Euler scheme is convergent
under fairly general conditions on $\bm f$, e.g. Lipschitz continuity, see \cite{atkinson2008introduction}.

The main result on numerical discretization below shows that
the composition of convergent discretizations can be used to
approximate flow maps.

\begin{theorem}[Temporal Discretization of Compositional Flow Map]

Consider a target mapping $\bm \varphi$ in the form of a composition of $l$ flow maps
$
    \bm \varphi
    =
    \bm \phi(\bm f^{l}, T_l - T_{l-1})
    \circ \cdots \circ
    \bm \phi(\bm f^{1}, T_1 - T_0)
$,
shortened as
\begin{equation}
    \bm \varphi  = \bm \phi_L \circ \cdots \circ \bm \phi_1.
\end{equation}

Suppose that for each $l$, $\hat{\bm \phi}(\bm f^l, \cdot)$
is a convergent numerical integrator, as defined in Definition~\ref{defn:convergence}.

Consider the following function with respect to a partition $\Delta: 0 = t_0 < t_1 <\cdots <t_s = T$
\begin{equation}
\label{eq:hat-varphi}
    \hat{\bm \varphi} (\cdot, \Delta)
    =
    \hat{\bm \phi}(\bm f^{l_s}, t_{s} - t_{s-1})
    \circ
    \cdots
    \circ
    \hat{\bm \phi}(\bm f^{l_1}, t_{1} - t_{0}),
\end{equation}
where $l_s = \min \{l : T_l > s \}$.
Then,
$\hat{\bm \varphi} (\cdot, \Delta) \to {\bm \varphi} (\cdot)$ uniformly on compact sets as $|\Delta| \to 0$.
\end{theorem}

As a consequence of this result, we can conclude that for any $\Hcalode$ possessing invariant UAP,
any discrete architecture corresponding to a convergent discretization inherits the UAP.

\begin{corollary}
    Suppose for a given control family $\Fcal$ and a terminal family $\Gcal$ that
    \begin{equation}
        \Hcalode := \{ \bm g \circ \bm \varphi : \bm g \in \Gcal, \bm \varphi \in \Acal_{\Fcal}\}
    \end{equation}
    possesses UAP.
    Moreover, suppose for each $\bm f \in \Fcal$,
    the numerical scheme $\hat{\bm \phi}_{\bm f}$ is convergent.
    We define the discretized attainable set
    \begin{equation}
        \hat{\Acal}_{\Fcal} :=
        \{
            \hat{\bm \phi}(\bm f^l, t_{l} - t_{l-1})
            \circ
            \cdots \hat{\bm \phi}(\bm f^1, t_1 - t_0)
            ~:~
            l\geq 1,
            0=t_0 < t_1 < \cdots < t_{l},
            {\bm f}^l \in \Fcal
        \}.
    \end{equation}
    Then, the discretized hypothesis space
    \begin{equation}
        \hat{\Hcal} := \{ \bm g \circ \bm \hat{\bm \varphi} : \bm g \in \Gcal, \hat{\bm \varphi} \in \hat{\Acal}_{\Fcal} \},
    \end{equation}
    corresponding to the deep neural network architecture
    $
        \bm x^{l+1}
        =
        \hat{\bm \phi}(\bm f^l, t_{l+1} - t_l)(\bm x^{l})
    $
    possesses the UAP.

    In addition, if  $\Gcal$ is $\sG$ invariant, $\Fcal$ is $\sG$ equivariant, and $\hat{\bm \phi}_{\bm f}$ is $\sG$ equivariant for each $\bm f \in \Fcal$, then $\hat{\Hcal}$ possesses the $\sG$ UAP.
\end{corollary}

\begin{remark}
    The first part of the above corollary extends the approximations results in \cite{li2019deep} for continuous dynamics to all convergent discretizations. The second part covers symmetry considerations.
\end{remark}


\begin{proof}
We prove an equivalent result, which does not require $t_s = T$,
and the convergence holds as $|\Delta| \to 0$ and $|t_s - T| \to 0$.
For convenience, we partition the product in \eqref{eq:hat-varphi} into
$L$ products $\bm \pi_L \circ \cdots \circ \bm \pi_1$,
where $\bm \pi_i$ is the composition of
$\hat{\bm \phi}(\bm f^{l_j}, t_j - t_{j-1})$
such that $l_j = i$.


Without loss of generality, we assume that $K = [-\kappa, \kappa]^n$.
We begin by induction on $L$.
We first prove the base case.
By assumption, it holds that $t_s < T$, then by definition,
for $\varepsilon$ there exists $\delta$,
such that
\begin{equation}
    |\hat{\bm \phi}(\bm f, T-t_s)\circ
    \hat{\bm \phi}(\bm f, \Delta) - \bm \phi(\bm f, T)|
    \le \varepsilon
\end{equation}
as $|\Delta| \le \delta$ and $|T - t_s| \le \delta$.
A direct calculation yields
\begin{equation}
    |\hat{\bm \phi}(\bm f, \Delta)\bm x- \bm \phi(\bm f, T)\bm x|
    \le \varepsilon +
    |(\hat{\bm \phi}(\bm f,T - t_s) - \operatorname{Id})
    \circ \hat{\bm \phi}(\bm x, \Delta)|.
\end{equation}
The result then follows from the definition of convergence.

Suppose that for $\varepsilon > 0$,
there exists a sufficiently small $\delta$, if $\max (t_{i+1} - t_i) \le \delta$.
Then,
\begin{equation}
    |\bm \pi_{L-1}\circ \cdots \circ \bm \pi_{1}(\bm x) - \bm \phi_{L-1}\circ \cdots \circ \bm \phi_{1}(\bm x)|
    \le \varepsilon.
\end{equation}

The key idea is to perform a telescope decomposition, i.e., for $\bm x \in K$, the estimation
\begin{equation}
     |(\bm \pi_L \circ \bm \Pi_L - \bm \phi_L \circ \bm \Phi_L) \bm x|
     \le
     |(\bm \pi_L \circ \bm \Pi_L - \bm \phi_L \circ \bm \Pi_L)\bm x|
     + |(\bm \phi_L \circ \bm \Pi_L - \bm \phi_L \circ \bm \Phi_L)\bm x|
\end{equation}
where
\begin{equation}
    \bm \Pi_L = \bm \pi_{L-1} \circ \bm \pi_{L_2} \circ \cdots \circ \bm \pi_{1},
\end{equation}
and
\begin{equation}
    \bm \Phi_L = \bm \phi_{L-1} \circ \bm \phi_{L-2}\circ \cdots \circ \bm \phi_{1}.
\end{equation}
By induction, we have $|\Pi_L\bm x - \Phi_L\bm x| \le \varepsilon < \kappa$.
Consider
\begin{equation}
    \tilde K := \{ \bm x : \operatorname{dist}(\bm x, \Phi_L(K)) \le \varepsilon\}.
\end{equation}
Then, by the definition of convergence on $\tilde K$, for sufficiently small $\delta$,
if $\max (t_{i+1} - t_i) \le \delta$, we have
\begin{equation}
    |(\pi_L \circ \Pi_L - \phi_L \circ \Pi_L)\bm x| \le \varepsilon.
\end{equation}
For the second term, just using the Lipschitz continuity of $\Phi_L$, then
\begin{equation}
    |(\phi_L \circ \Pi_L - \phi_L \circ \Phi_L)\bm x| \le \lip \Phi_L \varepsilon.
\end{equation}
Combining both we conclude the result.

\end{proof}

The applications detailed in Section~\ref{sec:application} will be presented
in continuous time.
Owing to the above result, all these density results still hold if one replaces the ODEs
by convergent discretizations, e.g. Euler, Runge-Kutta, etc.
For brevity, we will only present the continuous-time statements subsequently.

\subsection{Applications}

\label{sec:application}

In this section, we demonstrate how our proposed framework can be applied to obtain universal approximation
results of a variety of deep learning architectures designed to capture or preserve symmetry.
Some of these, such as the convolutional neural network, have been subjects
under intense study.
On the other hand, our framework can also be applied to study novel architectures for emerging
machine learning applications.

Before diving into application cases, we first discuss some advantages of our theoretical approach.
Unlike many other results on approximation theory of symmetric functions by neural networks (e.g.~\cite{yarotsky2018universal}),
the results here do not rely on the specification of an explicit network architecture,
much like the flavor of those in \cite{li2019deep}.
In fact, we show that a variety of different architectures satisfy the assumptions in our theory.
Second, more symmetry scenarios can be handled under our theory,
since the only assumption we made about the symmetry group $\sG$ is transitivity.
Besides shift invariance (convolutional networks) and full permutation invariance (set function networks),
we can also study other useful symmetry structures, such as the product permutation invariant structures (See Sec.~\ref{sec:application:product}).

Throughout this section, $\sigma$ denotes {an activation} function.
The only assumption on $\sigma$ is that $\chbar(\{ w\sigma(a \cdot + b) : w,a,b ,\in \R \})$
contains a one-dimensional well function $h$.
This is verified in \cite{li2019deep} that common activation functions like ReLu, Sigmoid and Tanh, satisfy this assumption.
For convenience, we will also choose a function family $\Gcal$ such that the covering condition 1 in
Theorem~\ref{thm:main} is satisfied:
$\Gcal$ is a $\sG$ invariant function family such that
for any compact set $K \subset \R^m$, there exists $\bm g \in G$ such that $K \subset \bm g (\R^n)$.
    For example, for scalar regression problems ($m=1$)
    we may choose $\Gcal = \{ g \}$ where $g(\bm x) = \sum_{i=1}^{n} x_i$
    to satisfy both invariance and the range covering condition.

\subsubsection{Shift Invariant Architectures}

We first discuss the approximation of shift invariant functions by dynamical systems driven by
shift equivariant control families.
A prime example of such architectures is the residual convolutional neural network.
We will show that our results immediately lead to universal approximation results for
shift-invariant functions using certain types of convolutional neural networks.

We recall the definition of the translation group in one or more dimensions.
The one dimensional case is relevant to sequence modelling applications (e.g. \cite{oord2016wavenet})
whereas the two/three dimensional cases apply to image and video analysis tasks.
In one dimension, the translation group is generated by the shift operator,
$\ft$ such that $\ft(i) = i+1$.
We then define $\sT = \sT_{1D} = \{ \ft, \ft^2,\cdots ,\ft^n\}$.
In two dimensions, we consider the case where the $n$ points are arranged in a 2D grid
of sizes $n_1,n_2$ with $n = n_1n_2$.
In image applications, each coordinate in this grid refers to a pixel in an image.
We may re-index the coordinates by the multi-index $(i_1,i_2)$,
where $i_1 = 1,2,\cdots, n_1$ and $i_2 = 1,2,\cdots, n_2$.
The two dimensional translation group $\sT_{2D}$ is generated by two shift operators:
\begin{equation}
    \begin{aligned}
        \ft_1( (i_1,i_2) ) = (i_1+1,i_2)
        \qquad
        \ft_2( (i_1,i_2) ) = (i_1,i_2+1)
        \\
        \sT_{2D} := \{ \ft_{1}^{j_1}\circ\ft_2^{j_2}: j_1 = 1,2,\cdots, n_1, ~~ j_2 = 1,2,\cdots, n_2\}.
    \end{aligned}
\end{equation}
The higher dimensional cases $\sT_{kD}$ follow similarly.

In one dimension where an input is a sequence represented as a vector,
the simplest continuous idealization of a convolutional neural network can be expressed as
\begin{equation}
    \dot{\bm x}(t) = v(t) \sigma(\bm w(t) \ast \bm x(t) + b(t) \bm 1),
\end{equation}
where for each $t$, $v(t),b(t)$ are scalars, $\bm w(t), \bm x(t), \bm 1$ are vectors and $\ast$ is the discrete
convolution operation, as defined in \eqref{eq:conv-1d} and \eqref{eq:conv-nd}.

In the general, high dimensional case, we define two variants of the idealization of residual convolution networks
\begin{equation}
    \Fcal_{\ast, 1} := \{v\sigma(\bm w \ast \bm x + b\mathbf 1)
    ~:~
    \bm w \in \mathbb{R}^{d_1\times d_2\times \cdots d_n}, v,  b \in \mathbb{R} \},
\end{equation}
and
\begin{equation}
\Fcal_{\ast, 2} := \{\bm w \ast \sigma(\bm x + b\mathbf 1)
~:~
\bm w \in \mathbb{R}^{d_1\times d_2\times \cdots d_n},  b \in \mathbb{R} \}.
\end{equation}
Note that these control families define a class of convolution-driven dynamics
where the size of the convolutional filters is equal to the size of the input signal.

Now, we may apply Theorem~\ref{thm:main2} to obtain the following result.
\begin{corollary}
    \label{cor:shift-invar-small}
    $\Hcalode(\Fcal_{\ast,1},\Gcal)$ and $\Hcalode(\Fcal_{\ast,2},\Gcal)$ possesses the $\sT_{kd}$ UAP.
\end{corollary}

{
The approximation ability of convolutional neural networks has been widely studied in the literature
\cite{
    yang2020approximation,bolcskei2019optimal,bao2019approx,zhou2020universality,
    yarotsky2018optimal,yarotsky2018universal,petersen2020equivalence}.
Thus, we should compare these with the results presented here.
First, we note that the shift invariance we discussed is in the index set.
Instead, \cite{yang2020approximation} uses a (non-symmetric) ReLU feed-forward neural network to approximate
shift-invariance mapping in $n$ dimensions via a wavelet-like construction.
Second, in each layer we use a periodic boundary condition
so that the analyzed architecture is exactly shift-invariant.
If zero boundary condition is adopted, then the architecture will not be strictly shift-invariant.
As a consequence, the boundary condition will deteriorate the interior symmetry structure when the network is deep enough.
For example, while \cite{zhou2020universality,oono2019approximation,bao2019approx}
studied approximation properties of convolutional neural networks,
due to the aforementioned boundary conditions the approximation results are established in the absence of symmetry requirements.
Third and most importantly,
our approximation results do not require a specific architecture,
and the restriction on the form of the controlled dynamics is minimal.
In contrast, for example, \cite{yarotsky2018universal} uses a intrinsic polynomial layer to prove a similar result.

It is important to emphasize that the convolutional control families we consider here have filters sizes equal to the input
or hidden state sizes, while requiring only 1 channel.
Thus, this is different from many of the aforementioned settings where small filters or multiple channels may be considered.
The reason we apply our results to this simple setting is because we focus on symmetry considerations, and these are
in a sense minimal constructions of convolutional networks that can achieve universal approximation of shift-invariant function.
However, the idea here can be generalized to filters with variable lengths.
These require additional technical arguments and will be addressed in future work.

Lastly, we note that \cite{petersen2020equivalence} built a connection between fully connected neural networks
and convolutional neural networks in order to translate approximation results for one to the other.
There, the authors exploit a suitable transformation representative,
which corresponds to the $\coor$ operator we introduced in this paper.
In this sense, Theorem~\ref{thm:main} extends the notion of representatives to
deduce universal approximations under more general types of symmetries, including but not limited to shift symmetry.
Most previous works on the approximation theory of CNNs focus on non-residual convolutional neural networks.
To the best of our knowledge, Corollary~\ref{cor:shift-invar-small} gives a first
guarantee of universal approximation of shift-invariant functions using continuous-time (residual) CNNs.
}

\subsubsection{Full Permutation Invariant Architectures}
\label{sec:applications:symmetric}

Besides the well-known convolutional neural network,
we demonstrate that our framework can be used to derive universal
approximation results under other symmetry settings.
Here, we focus on functions possessing full permutation invariance, i.e.
those invariant to arbitrary rearrangement of the coordinates of their input vectors.
These are also known as ``set functions'',
since their outputs only depend on the collection of input coordinate values
but not on their order of arrangement.
Learning functions that possess such invariance structures is important in many applications,
including population statistics, anomaly detection, cosmology \cite{NIPS2017_f22e4747}.
One can see that this invariance requirement is strong, and thus structures obeying such invariance
tend to be limited in flexibility in approximation.
As motivated earlier, composition or dynamics provides a way to expand complexity of hypothesis spaces
while respecting possibly very restrictive invariances.
Our goal is to prove a general sufficient result for universal approximation of full permutation invariant functions
through dynamics, thereby providing a guidance to practical construction of neural network architectures
respecting these symmetries.

In our framework, the symmetry group under consideration is the full permutation group, i.e. $\sG = \sS$.
Recall that in this case, we only need to construct a control family that verify conditions 2,3 in Theorem~\ref{thm:main}.
One can show that any of the following choices are sufficient:
\begin{align}
    \Fcal_{s} :=&
    \{
        a \sigma(w\bm x + v \Sigma\bm x + b\bm 1)
        :
        a, w, v,b \in \R
    \}, \label{eq:Fs1}\\
    \Fcal_{s} :=&
    \{
        a\sigma(w\bm x + c \bm 1) + b\Sigma(\sigma(v\bm x + d\bm 1))
        :
        a, b, w, v,d \in \R
    \}, \label{eq:Fs2}
\end{align}
where $\Sigma$ is a matrix of all 1s, therefore $\Sigma\bm x = (x_1+\cdots + x_n) \bm 1$.


Most existing work on the universal approximation of permutation invariant functions
by neural networks focus on shallow or non-residual architectures \cite{sannai2019universal, maron2019universality,ravanbakhsh2020universal}.
Thus, approximation there is achieved by allowing width to be large enough.
In contrast, we investigate how to build approximators with symmetry through composition or dynamics.
Our method shows that for a fixed width,
many constructions of the layer architecture can achieve universal approximation in this manner.
We highlight that our theory can handle the permutation invariance generated by both ``intrinsic'' equivariant layers
in the sense discussed in \cite{sannai2019universal}, or by averaging in \cite{ravanbakhsh2020universal,yarotsky2018universal}.

Concretely, we show that our results show that many popular
permutation invariant architectures can achieve universal approximation
at \text{fixed width} in its deep, residual form.
In fact, the control families~\eqref{eq:Fs1}-\eqref{eq:Fs2}
can be viewed as residual versions of the DeepSets~\cite{NIPS2017_f22e4747} and PointNet~\cite{charles_pointnet_2017} architectures
(see also~\cite{germain_deepsets_2022} for applications to PDEs),
and we can verify that they indeed induce uniform approximation
through composition. One caveat is that in these works, it is sometimes
suggested to use the $\max$ operator in place of the $\Sigma$ operator.
Our results do not currently apply to this case.
For example, we can show that the variant
$
    \Fcal_{s,\max} := \{v \sigma(a\bm x +b \max(\bm x)\bm 1+ c\bm 1 ):v,a,b,c \in \R\}
$,
where $\max(\bm x) = \max_i x_i$
does not possess $\sS$ UAP for arbitrary choice of terminal family
satisfying the conditions of Theorem~\ref{thm:main}.
Concretely, we may choose $m=1$, $\Gcal = \{g\}$, where
$ g(\bm x) = \max(\bm x)$.
Then, for any $\bm x, \bm y \in Q$ such that $x_1 > y_1$,
one can show that for any $\bm \varphi \in \Acal_{\Fcal_{s,\max}}$,
it holds that $\bm \varphi(\bm x), \bm \varphi(\bm y) \in Q$
and $[\bm \varphi(\bm x)]_1 > [\bm \varphi(\bm y)]_1$.
Therefore, this family cannot approximate functions such as
$\min(\bm x) = \min_{i} x_i$.
This argument does not rule out universal approximation
for other choices of terminal families.

As a further application of our results,
we can show that the Janossy pooling type of variants introduced in \cite{wagstaff2022universal}
can also induce universal approximation in its residual form.
This corresponds to the following control families
(order 1 and order 2 Janossy pooling, respectively)
\begin{align}
    \Fcal_{s,Ja,1} &:= \{ v\sigma(a\bm x + b\sum_{i=1}^n \phi(x_i) + c\bm 1): v,a,b,c \in \R\}, \label{eq:janossy_1}\\
    \Fcal_{s,Ja,2} &:= \{ v\sigma(a\bm x + b\sum_{i=1}^n \phi(x_i,x_j) + c \bm 1): v,a,b,c \in \R\},
    \label{eq:janossy_2}
\end{align}
where the $\phi$ in each case is some chosen scalar valued function.
We can verify that $\Fcal_{s,Ja,1}$ induces the $\sS$ UAP
if $\phi$ is Lipschitz and there exists a neighborhood $U$ of $\phi(0)$
such that $\phi^{-1}(U)$ is bounded.
In particular, any sigmoid function satisfy this condition.
Similarly, $\Fcal_{s,Ja,2}$ induces $\sS$ UAP
if $\phi$ is symmetric, Lipschitz and $\phi(z,0)$ satisfies the previous condition.

Thus, we arrive at the following corollary that shows that all
these architectures in their residual form possesses $\sS$ UAP,
and checking them is a simple application of our results.
The verification details are presented in Section~\ref{sec:proofs:verification}.

\begin{corollary}
    \label{cor:symmetric}
    Choosing $\Fcal_s$ as any of Eq.~\eqref{eq:Fs1}-\eqref{eq:janossy_2},
    $\Hcalode(\Fcal_s, \Gcal)$ possesses the $\sS$ UAP.
\end{corollary}

\begin{remark}
    It is known that when the hidden dimension
    is too small, then universal approximation through PointNet/DeepSet
    and their variants are not possible, see \cite[Theorem 20]{wagstaff2022universal}.
    However, we consider in this paper residual structures,
    which do not appear to have such a constraint.
    It is also interesting to see if the above results can
    be extended to probabilistic settings such as those
    considered in~\cite{bloem-reddy_probabilistic_2020}
    and the combination of sum-based and LSTM-based
    symmetric architectures proposed in~\cite{vinyals_order_2015}.
\end{remark}

\subsubsection{Product Permutation Invariant Architectures}
\label{sec:application:product}
Besides the full symmetric group,
one is often interested in functions which are invariant
to certain subgroups of $\sS$ distinct from those generated by simple shifts.
These symmetry requirements arise naturally in a number of applications in computational chemistry
and materials science.
As an example, to specify a crystal lattice consisting of a number of atomic species,
a convenient representation is in the form of the crystallographic information file (CIF)~\cite{hall1991crystallographic}.
Examples of partial features under
the CIF representations of crystal lattices are shown in Fig~\ref{fig:crystal_latt}.
Observe that the rows and columns can be rearranged
without affecting the essential structure it represents.
Thus, quantities (energy, band-gap) that depend on such structures
are invariant to these permutations as well.
This can be effectively used for property
prediction~\cite{xieCrystalGraphConvolutional2018,chenGraphNetworksUniversal2019}
or inverse design~\cite{ren2020inverse}.

\begin{figure}[!h]
    \centering
    \includegraphics[width=\textwidth]{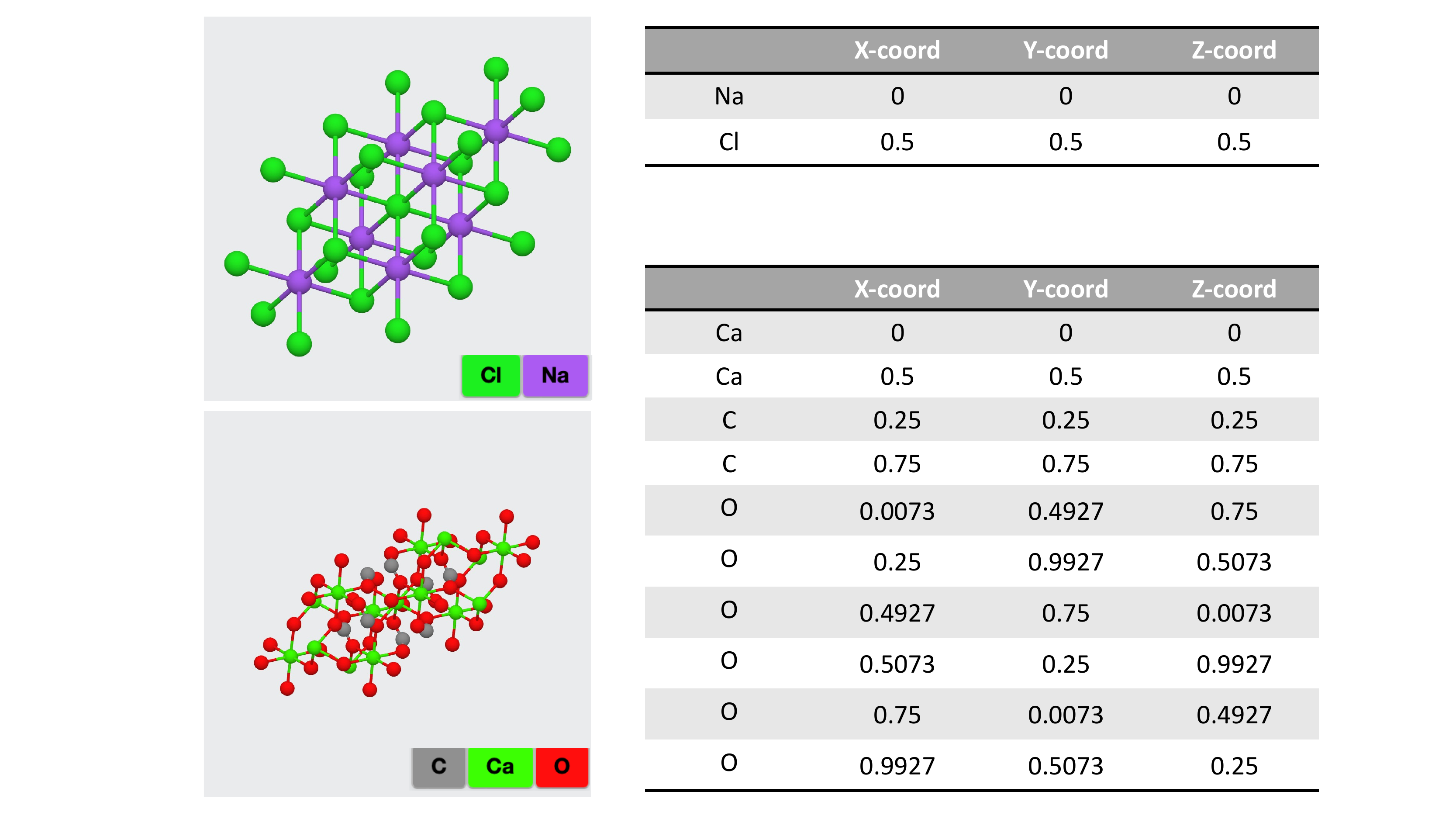}

    \caption{
        Structure of sodium chloride (top) and calcium carbonate (bottom) in CIF form.
        Coordinates are fractional,
        relative to the  crystal lattice vectors and
        atoms are listed in each unit cell.
        Observe that the rows and columns can be rearranged without affecting
        the structure it represents, since the order of the atoms
        and the order of the relative coordinate axes do not matter,
        provided the property of interest are rotation and reflection symmetric.
        Thus, quantities (density, energy, band-gap) that depend on such structures
        are invariant to product permutation in two dimensions, at least on these
        coordinate features.
        Graphics are obtained from the materials project~\cite{osti_1199028,osti_1207669}.
    }
    \label{fig:crystal_latt}
\end{figure}

If we flatten the input data into a vector of dimension $n$,
such permutations forms a subgroup of $\sS$.
We call this group a {\it product permutation group}
and functions invariant with respect to it are collectively referred to
as {\it product permutation invariant functions}.
We now give their precise definitions.

Consider the direct product of permutation acting on a double index $(i,j)$
\begin{equation}
(\fa,\fb)(i,j) = (\fa(i),\fb(j)).
\end{equation}
The collection of all $(\fa,\fb)$ forms the $2D$
product permutation group,
denoted as $\sS_{n_1}\times \sS_{n_2}$, or simply $\sS_{2D}$,
which is a subgroup of $\sS$ where $n=n_1n_2$.
Properties such as crystal structure induced ones
are invariant with respect to this permutation group if the data is
presented as in Fig.~\ref{fig:crystal_latt}.

Now, a simple way to build a control family that is
equivariant with respect to the $2D$ product permutation group is
\begin{equation}
    \label{eq:F-s-2D-1}
    \Fcal_{s,2D,1} = \{ v \sigma(w_0 \bm x + w_{r,1} \Sigma_{r,1}\bm x + w_{c,1} \Sigma_{c,1} \bm x +  c\bm 1): v,w_{0}, w_{r,1},w_{c,1}, c \in \mathbb R\}.
\end{equation}
Here $\Sigma_{r,1}$ and $\Sigma_{c,1}$ denotes the row and column sum of a tensor $\bm x$,
\begin{equation}
    [\Sigma_{r,1}\bm x]_{i,j} = \sum_{j'} x_{i,j'}.
\end{equation}
\begin{equation}
    [\Sigma_{c,1}\bm x]_{i,j} = \sum_{i'} x_{i',j}.
\end{equation}
Furthermore, one can consider the second order variant
\begin{equation}
    \label{eq:F-s-2D-2}
    \begin{split}
    \Fcal_{s,2D,2} = \{ v \sigma(w_0 \bm x + w_{r,1} \Sigma_{r,1}\bm x + w_{c,1} \Sigma_{c,1} \bm x +  w_{r,2} & \Sigma_{r,2}\bm x + w_{c,2} \Sigma_{c,2} \bm x+  c\bm 1 ) :\\ &v, w_{0}, w_{r,1}, w_{r,2}, w_{c,1}, w_{c,2},c \in \R\},
    \end{split}
\end{equation}
where
\begin{equation}
    [\Sigma_{r,2}\bm x]_{i,j} = \sum_{j',j''} x_{i,j'}x_{i,j''},
\end{equation}
\begin{equation}
    [\Sigma_{c,2}\bm x]_{i,j} = \sum_{i',i''} x_{i',j}x_{i'',j}.
\end{equation}
Clearly, it holds that $\Fcal_{s,2D,1} \subset \Fcal_{s,2D,2}.$,
and hence universal approximation of the former implies that
of the latter.
We can use our results to deduce the following UAP
propoerty for these architectures.

\begin{corollary}
\label{cor:product-permutation}
The dynamical hypothesis spaces
$\Hcalode(\Fcal_{s,2D,1}, \Gcal)$
and thus
$\Hcalode(\Fcal_{s,2D,2}, \Gcal)$
possess the $\sS_{2D}$ UAP.
\end{corollary}

The above result can be generalized to higher dimensions.
Consider the $kD$ product permutation group
$\sS_{kD} := \sS_{n_1} \times \sS_{n_2} \times \sS_{n_k}$,
where the element $\bm \fa = (\fa_1,\cdots,\fa_k)$ acts on
a multi-index $\bm i = (i_1,\cdots,i_k)$ by
\begin{equation}
    \bm \fa(\bm i) = (\fa_1(i_1),\cdots,\fa_k(i_k)).
\end{equation}
We first extend $\Sigma_{r,1}$ and $\Sigma_{c,1}$ to higher dimensions.
Define
\begin{equation}
    [\Sigma_{s,1}\bm x]_{\bm i} = \sum_{\bm j: j_s = i_s} x_{\bm j},
\end{equation}
and similarly,
\begin{equation}
    [\Sigma_{s,2}\bm x]_{\bm i} = \sum_{\bm j, \bm j': j_s = j_s' = i_s} x_{\bm j}x_{\bm j'}.
\end{equation}
Next, we define
\begin{equation}
    \label{eq:F-s-kd-1}
    \Fcal_{s,kD,1} = \{ v \sigma(w_0 \bm x + \sum_{s=1}^k w_{s,1} \Sigma_{s,1}\bm x +  c\bm 1) : v, w_0, w_{s,1}, c \in \R \},
\end{equation}
and
\begin{equation}
    \label{eq:F-s-kd-2}
    \Fcal_{s,kD,2} = \{ v \sigma(w_0 \bm x + \sum_{s=1}^k w_{s,1} \Sigma_{s,1} \bm x + w_{s,2} \Sigma_{s,2}\bm x +  c\bm 1): v, w_0,w_{s,1}, w_{s,2},c \in \R\}.
\end{equation}
Then, an analagous approximation results below holds for high-dimensional product
permutation symmetric functions.
The proof is identical to the $2D$ case and is hence omitted.
\begin{corollary}
\label{cor:product-permutation_nd}
The dynamical hypothesis spaces
$\Hcalode(\Fcal_{s,kD,1}, \Gcal)$
and thus
$\Hcalode(\Fcal_{s,kD,2}, \Gcal)$
possess the $\sS_{kD}$ UAP for $k\geq 2$.
\end{corollary}

To the best of our knowledge,
there is no study on the approximation theory
of residual architectures under product permutation invariance.
However, as discussed earlier,
such data symmetry structures feature in a wide variety of applications,
especially in computational chemistry and physics.
One way to deal with these symmetries in the practical literature,
at least in the case of lattice structures,
is to appeal to graphical representations
\cite{xieCrystalGraphConvolutional2018,chenGraphNetworksUniversal2019},
which extracts permutation invariant features that can be subsequently
used to predict desired material properties.
Here, the network architecture proposed are more general in the sense that
it does not rely on having an explicit graphical representation,
i.e. the data need not contain spatial coordinate information.
Instead, deep neural networks are constructed to satisfy these symmetries in a intrinsic manner.
Corollaries~\ref{cor:product-permutation} and~\ref{cor:product-permutation_nd}
establishes a first basic approximation guarantee of these networks for
modelling functions symmetric under product permutation.
A concurrent paper~\cite{tian2022tbd} studies the application
of these novel architectures for modelling structure-property
relationships in crystalline and amorphous materials.
The latter does not admit efficient graphical representations
due to disorder, and are thus less amenable to graphical
approaches.
\section{Technical Details}
\label{sec:proofs}

This section includes the postponed proofs in Section~\ref{sec:thm}.
It is useful to first give an overall sketch of the key ideas in
establishing our approximation results.

\subsection{Proof Outline}

The basic framework for approximating a invariant function $F$
by composition is to first fix an invariant function
$g$ whose range covers that of $F$.
Then, it remains to construct an equivariant
mapping $\varphi:\R^n\rightarrow\R^n$ so that
$F \approx g \circ \varphi$.
One can show that under fairly general conditions, the existence of $\varphi$
is guaranteed (Proposition~\ref{thm:level-set})

Thus, it remains to determine sufficient conditions on the control family
$\Fcal$ such that its attainable set $A_\Fcal$ can approximate arbitrary
equivariant mappings.
As in~\cite{li2019deep}, one can reduce this problem to matching
an arbitrary finite point set in $\R^n$ to another using flows
of the dynamical system.
There, the crucial property of well functions enables this:
the zero set of the well function can keep certain point fixed,
whereas the non-zero part can move points.
This then achieves the required point matching property through
repeated composition of carefully constructed flow maps.

The key difference when extending this argument to the symmetric
setting is that the induced motion on each coordinate component of a point
is no-longer independent.
For example, if $\sG=\sS$, then all coordinates of each point will
experience the same transformation.
Hence, the requirement for arbitrary point set matching must be
suitably relaxed.
It turns out that we only need to be able to match point sets
consisting of points belong to distinct $\sG$ orbits,
greatly relaxing the type of constructions required.
This is because, within the same orbit, we can use the invariant property
of $g$ and the equivariant property of $\Fcal$ to transport the point to the
desired location.
Theorem~\ref{thm:framework} establishes such conditions that are sufficient
to induce universal approximation under symmetry settings.

The rest of the results then show how the relaxed point matching property
can be induced by the attainable set $A_\Fcal$.
Theorem~\ref{thm:approx-dynamical} shows that the existence
of suitable types of well functions ensures this
in the case $\sG=\sS$, i.e. full permutation symmetry.
Then, Theorem~\ref{thm:approx-dynamical-general} extends
this to any transitive subgroup $\sG\le \sS$ by enforcing
additional conditions on the control family $\Fcal$ -
it should not be ``too symmetric'' so as to
lose the ability to drive points in a general way.
These condition is summarized in Definition~\ref{defn:resolvence}.
Namely, we need to have controls that can cross boundaries of
cross-sections (vectors in $\R^n$ with repeated coordinate values).
Conditions such as direct connectivity (Definition~\ref{defn:direct_connectivity})
serve to ensure this.

The last part of this section (Sec.~\ref{sec:proofs:verification})
contains verifications that the discussed architectures in
Section~\ref{sec:application} satisfy the assumptions of the approximation Theorems
~\ref{thm:approx-dynamical} or~\ref{thm:approx-dynamical-general}.

Having introduced the overall proof sketch, now we present step by step the
technical results.

\subsection{Approximation Framework}

We start by introducing the basic framework for
the approximation theory about invariant and
equivariant functions via composition or dynamics.
Hereafter, we fix $\sG \le \sS$ as a transitive subgroup,
and $\sG$ invariance (resp. $\sG$ equivariance)
will be shortened as invariance (resp. equivariance).

For simplicity, in this section we only consider the case $ m =1 $.  Namely, we
only consider the scalar regression setting where the approximation
target is a function $F:\R^n \to \R$.
The case $m>1$ follows similarly as in the construction in \cite[Proposition
3.8]{li2019deep}, and is hence omitted for simplicity.  We begin with the
following principle, which indicates that approximating an invariant function
can be decomposed into two steps:
\begin{enumerate}
    \item Choose a simple invariant function
    $g:\R^n \rightarrow \R$, for example, $g(\bm x) = \sum_{i=1}^n x_i$.
    \item
    Construct an equivariant hypothesis set
    $\Acal$ consisting of equivariant mappings on $\R^n$ to itself
    with enough complexity.
\end{enumerate}
Then, the set of functions
$\{ g\circ \bm \varphi : \bm \varphi \in \Acal \}$
can serve as universal approximators for invariant functions.
In the context of neural networks, $g$ will be the final layer whereas
$\bm \varphi \in\Acal$ consists of a stack of intermediate layers.
Theorem~\ref{thm:approx-dynamical} and
Theorem~\ref{thm:approx-dynamical-general} to be presented later
will deal with how the required equivariant function $\bm \varphi$ can be constructed from compositions.

Proposition~\ref{thm:level-set} below establishes that this two-step decomposition scheme is sufficient for approximation.

\begin{proposition}[Approximating Equivariant Mapping is Sufficient]
    \label{thm:level-set}
    Suppose $F\in C(\mathbb R^n)$ is an invariant function.
    Let $g$ be a Lipschitz continuous and invariant function with
    $F(\R^n) \subset g(\R^n)$.
    Then, for any compact $K\subset \R^n$ and $\varepsilon>0$,
    there exists an equivariant mapping $\bm \varphi \in C(\R^n;\R^n)$
    such that
    \begin{equation}
        \| F - g \circ \bm{\varphi}\|_{L^p(K)} \le \varepsilon.
    \end{equation}
\end{proposition}

To complete this proof, the following lemma is required.
It tells us that a $\sG$ invariant set can be decomposed into the disjoint union of cross sections,
up to a small set.
Recall the definition of the cross section $Q_{A}$ as introduced in \eqref{eq:cross-section-A}:
\begin{equation}Q_A = \cup_{\fa \in A} Q_{\fa},\qquad \text{ where } Q_{\fa} = \fa(Q).\end{equation}
\begin{definition}[General position]
    \label{defn:general-position}
    We say a vector $\bm x$ is in general position if all of its coordinates are distinct.
\end{definition}

\begin{lemma}[Partitioning the space via cross sections]
    \label{lem:cross-section}
    Given $\sG \le \sS$, the following holds
    for any $\sG$ transversal $A$:
    \begin{enumerate}
    \item For two distinct $\fg,\fg' \in \sG$, we have $\fg(Q_{A}) \cap \fg'(Q_{A}) = \emptyset$.
    \item $\displaystyle \R^n \setminus \bigcup_{\fg \in \sG} \fg(Q_{A})$ is
    the set of points that are not in general position, and thus of zero Lebesgue measure.
    \end{enumerate}
    Consequently, if $K$ is a $\sG$ invariant set,{(i.e., $\fg(K) = K$ for all $\fg \in \sG$)},
    then $K \setminus \cup_{\fg \in \sG} (K\cap\fg(Q_{A}))$ is of zero Lebesgue measure.
\end{lemma}

\begin{proof}
    Suppose $\bm x \in \fg(Q_A) \cap \fg'(Q_A)$.
    Then, by definition there exists $\fa \neq \fa' \in A$,
    such that $\bm x \in Q_{\fg\fa} \cap Q_{\fg'\fa'}$.
    The structure of $Q_{\fa}$ immediately tells us that $\fg \fa = \fg' \fa'$.
    Since $A$ is a transversal of $\sG$,
    we have $\fg = \fg'$ and $\fa = \fa'$, leading to a contradiction.

    On the other hand, we claim that if $\bm x$ is in general position,
    then there exists $\fg$ such that $\bm x \in \fg(Q_A)$.
    The construction is straightforward:
    since $\bm x \in Q_{\fa}$ for some $\fa \in \sS$,
    then by the choice of $A$, there exists a decomposition $\fa = \fg \fb$,
    such that $\fb \in A$ and $\fg \in \sG$. Therefore, we have $\bm x \in \cup_{\fg \in \sG} \fg(Q_A)$. The reverse inclusion holds trivially.
\end{proof}

{Now we are ready to prove Proposition~\ref{thm:level-set}. The idea is simple. We first obtain $\bm u$ without the equivariance constraint, this is done by \cite[Theorem 3.8]{li2019deep}. The core of the following proof is to modify the ${\bm u}$ into our desired $\bm f$.}

\begin{proof}(Proof of Proposition~\ref{thm:level-set})
    Without loss of generality, we assume that $K$ is a $\sG$ invariant set,
    otherwise we can enlarge $K$ to make it $\sG$ invariant.
    Choose the $\sG$ transversal $A$ arbitrarily,
    it follows from Lemma~\ref{lem:cross-section} that
    $K = \cup_{\fg \in \sG} (K \cap \fg(Q_{A}))$ up to a measure zero set.
    Define $\varepsilon' := \frac{\varepsilon}{|\sG|(1+\lip g)}$,
    by results in \cite[Theorem 3.8]{li2019deep},
    for any $\varepsilon' > 0$ there exists $\bm{u}$ such that
    \begin{equation}
        \|F - g \circ \bm{u} \|_{L^p(K)} \le \varepsilon'.
    \end{equation}
    Note that $\bm{u}$ here is not necessarily equivariant, otherwise we are done.
    Now we attempt to find $\bm f$ by some kind of equivariantization on $\bm u$ as explained below.
    Since $\bm u $ is in $L^p$,
    we consider a compact set $O \subset Q_A$
    such that $\|\bm{u}\|_{L^p(Q_{A} \setminus O)} \le \varepsilon'$.
    Take a smooth truncation function $\chi \in C^{\infty}(\R^d)$,
    whose value is in $[0,1]$,
    such that $\chi|_O = 1$ and $\chi|_{Q_{A}^c} = 0$.

    For $\bm x \in \fg( Q_{A})$ with $\fg \in \sG$,
    define $\bm{f}(\bm{x}) = \fg(\tilde{\bm{u}}(\fg^{-1}(\bm x)))$,
    where $\tilde{\bm{u}} = \chi\bm{u}$ is smoothed truncated version of $\bm u$.
    Since different $\fg(Q_{\sG})$ are disjoint,
    the value of $\bm{f}$ is unique in $\cup_{\fg \in \sG}(K \cap \fg(Q_{\sG}))$.
    We set $\bm{f}(\bm x) = 0$ in the complement of $\cup_{\fg \in \sG}(K \cap \fg(Q_{A}))$.
    The truncation function $\chi$ ensures that $\bm{f}$ vanishes on the boundary of $Q_A$, therefore $\bm f$ is continuous,
    and direct verification shows that $\bm{f}$ is $\sG$ equivariant.

   It then suffices to estimate $\|F - g \circ \bm f\|_{L^p}$, since both $F$ and $g \circ \bm f$ are equivariant, it is natural and helpful to restrict our estimation on $K \cap Q_A$, since
    \begin{equation}\|F - g \circ \bm f\|_{L^p(K)} = |\sG| ~ \|F - g \circ \bm f \|_{L^p(K\cap Q_A)}.\end{equation}
To estimate the error on $K\cap Q_A$, we first estimate the error $\|\bm u - \bm f\|_{L^p(K\cap Q_A)}$. Since $\bm u $ and $\bm f|_{Q_A} = \tilde {\bm u}$ coincide on $O$, we have

    \begin{equation}
        \begin{split}
            \|\bm{u} - \bm{f}\|_{L^p(K \cap Q_{A})} & = \|\bm u - \tilde{\bm u}\|_{L^p(K \cap Q_A)} \\
            & \le \| \bm u \|_{L^p(Q_{A} \setminus O)} = \varepsilon'.
        \end{split}
    \end{equation}
The inequality holds from $\chi$ takes value in $[0,1]$.
    Since $g$ is Lipschitz,
    we have $\|g\circ \bm{u} - g\circ \bm{f}\|_{L^p(K\cap Q_{A})} \le \lip g \varepsilon'$, yielding that $\|F - g\circ\bm{f}\|_{L^p(K\cap Q_{A})} \le (1+\lip g) \varepsilon' $. We finally have $\|F - g\circ\bm{f}\|_{L^p(K)} \le (1+\lip g)|\sG| \varepsilon' = \varepsilon$.
\end{proof}

\subsection{Universal Approximation under Symmetry using Compositions}
\label{sec:thm:composition}

In this part, we discuss how to achieve universal approximation of $\sG$ equivariant mappings through composition.
We begin with the following definitions,
which can be regarded as a summary and generalization of those introduced in \cite{li2019deep}.
Roughly speaking, our strategy is to show that under mild conditions,
the universal approximation property can be reduced to transporting a finite point set $X$ to another finite point set $Y$,
with $X$ and $Y$ having the same cardinality.
However, this result does not hold if no additional requirement is imposed on $X$ and $Y$.
For example, for some $\fg \in \sG$, if $\fg(\bm x) = \bm x'$,
then the $\bm f(\bm x') = \fg(\bm f(\bm x))$ provided that $\bm f$ is $\sG$ equivariant,
imposing a constraint on $Y$.
To this end, we introduce the concept of \emph{$\sG$ distinctness},
restricting the position of the finite point set $X$ to obey such constraints.

\begin{definition}[$\sG$ distinctness]
    \label{defn:G-distinctness}
    A point set $X = \{\bm{x}^1,\cdots,\bm{x}^M\}$ is called \emph{$\sG$ distinct},
    if $\fg(\bm{x}^i) = \bm{x}^j$ implies $\fg = \fe$ and $i = j$,
    with $\fe$ being the identity element.
    In other words, a point set is $\sG$ distinct
    if and only if the $\sG$ orbit of the points are distinct.
\end{definition}


Now, we are ready to state and prove a basic result on compositional approximation under symmetry conditions,
which gives general sufficient conditions for building a complex hypothesis spaces out of potentially
simple ones through function composition.

\begin{theorem}[Universal approximation of invariant functions via composition]
    \label{thm:framework}
    Let $F \in C(\R^n)$ be invariant.
    Suppose $\Gcal$ is a family of invariant functions and
    $\Acal$ is a family of equivariant Lipschitz mappings with the following properties:
    \begin{enumerate}
        \item For any compact $K\subset \R^n$,
        there exists a Lipschitz $g \in \Gcal$ such that $F(K) \subset g(\R^n)$.
    \item
    $\Acal$ is closed under composition,
    i.e., if $\bm{f_1} \in \Acal$ and $\bm{f_2} \in \Acal$,
    then $\bm{f_1} \circ \bm{f_2} \in \Acal$.
    \item(Coordinate zooming)
    For any increasing function $v \in C(\R)$,
    compact interval $\mathbb I \subset \R$ and tolerance $\varepsilon > 0 $,
    there exists $u$ such that $u^{\otimes} \in \Acal$ and $\|u - v\|_{C(\I)} \leq \varepsilon$.
    \item(Point matching)

    For $M >0$, a transversal $A$ of $\sG$,
    a $\sG$ distinct point set $\{\bm{x}^1,\cdots,\bm{x}^M\}$ with
    $\bm{x}^i \in Q_A$, another point set $\{\bm{y}^1,\cdots,\bm{y}^M\} \subset Q_A$ and tolerance $\varepsilon>0$,

    {there exists $\bm f \in \Acal$ such that}
    \begin{enumerate}
        \item For $\bm x^i$ in general position, we have $|\bm f(\bm x^i) - \bm y^i| \le \varepsilon$.
        \item For $\bm x^i$ not in general position, we have $|\bm f(\bm x^i)| \le 1$.
    \end{enumerate}
\end{enumerate}
    Then, for any compact set $K \subset \R^n$, tolerance $\varepsilon>0$ and $p \in [1,\infty)$,
    there exists ${\bm{\varphi}} \in \Acal$ and $g \in \Gcal$ such that
    $\|F - g\circ\bm{\varphi}\|_{L^p(K)} \le \varepsilon$.
\end{theorem}

\begin{proof}[Proof of Theorem~\ref{thm:framework}]
    Without loss of generality, we assume $K = [-\kappa,\kappa]^n$ is a hyper-cube centered at the origin.
    Select $Q_A$ for arbitrary $\sG$ transversal $A$.
    Since $K$ is $\sG$ invariant,
    the decomposition  $K = \cup_{\fg \in \sG} \fg(K \cap \fg(Q_{A}))$ holds up to a measure zero set,
    according to Lemma~\ref{lem:cross-section}.

    Observe that $K \cap Q_{A}$ is a polyhedron.
    In the following, we denote by $\bm \lambda$ the Lebesgue measure.

    \paragraph{Step 1.}
    By {Proposition~\ref{thm:level-set}},
    we can find {an invariant} $g \in \Gcal$
    and {an equivariant} $\bm \varphi \in C(\R^n,\R^n)$ such that
    \begin{equation}
        \|F - g\circ{\bm \varphi}\|_{L^p(K)} \le {\varepsilon/2}.
    \end{equation}
    Now we consider a piecewise constant approximant $\bm \varphi_0$.
    Given a scale $\delta > 0$, consider the grid $\delta\Z^n$ with size $\delta$.
    Let $\bm{i} = [i_1,\cdots,i_n] \in \Z^n$ be a multi-index,
    and $\chi_{\bm i}$ be the indicator of the cube
    \begin{equation}
        \square_{\bm i,\delta} := [i_1\delta,(i_1+1)\delta]\times \cdots \times [i_n\delta,(i_n+1)\delta].
    \end{equation}
    Since $\bm{\varphi}$ is in $L^p(K)$,
    by standard approximation theory $\bm{\varphi}$
    can be approximated by equivariant piecewise constant (and $\sG$ equivariant) functions
    \begin{equation}
        \bm{\varphi}_0(\bm x) = \sum_{\bm{i}} \bm{y}_{\bm i} \chi_{\bm i}(\bm x),
    \end{equation}
    where
    \begin{equation}
        \bm y_{\bm i}
        =
        \bm\lambda(\square_{\bm i,\delta})^{-1} \int_{\square_{\bm i,\delta}} \bm \varphi(\bm x) d\bm x
    \end{equation}
    is the local average value of $\bm \varphi$ in $\square_{\bm i,\delta}$.
    Then, we have
    \begin{equation}
        \|\bm \varphi - \bm \varphi_0\|_{L^p(K)} \le \omega_{\bm{\varphi}}(\delta)
        [\bm\lambda(K)]^{1/p} \to 0\end{equation} as $\delta \to 0
    $,
    where
    $\omega_{\bm \varphi}$ is the modulus of continuity (restricted in the region $K$), i.e., \begin{equation}\omega_{\bm \varphi}(\delta)  := \sup_{|\bm x - \bm y| \le \delta} |\bm \varphi(\bm x) - \bm \varphi(\bm y)|\end{equation} for $\bm x$ and $\bm y$ in $K$
    and $\bm\lambda(K)$ is the Lebesgue measure of $K$.
    Since $g$ is $\sG$ invariant,
    we can replace $\bm y_i$ by arbitrary $\fg(\bm y_i)$ for $\fg \in \sG$.
    Therefore, we can assume without loss of generality that $\bm y \in Q_A$.
    Thus,
    \begin{equation}
        \label{eq:step-1-estimation}
        \begin{split}
        \|F - g \circ {\bm \varphi_0}\|_{L^p(K)}
        \le~ & \| F - g\circ \bm \varphi \|_{L^p(K)} + \| g \circ \bm \varphi - g \circ \bm \varphi_0\|_{L^p(K)} \\ \le~ &
        \varepsilon/2 + \lip(g) \omega_{\bm{\varphi}}(\delta)[\bm \lambda(K)]^{1/p}.
        \end{split}
    \end{equation}
    Choose suitable $\delta>0$ such that the right hand side of \eqref{eq:step-1-estimation}
    is smaller than $\varepsilon$.

    \paragraph{Step 2.}
    Let $\bm p_{\bm i} = \bm i\delta = [i_1\delta, \cdots,i_n\delta]$ be the a vertex of $\square_{\bm i,\delta}$.
    Define $\Ical$ as the maximal subset of
    $\{ \bm i:\bm p_{\bm i} \in \overline{Q_A}\}$
    such that $\mathcal{P} := \{ \bm p_{\bm i} : \bm i \in \Ical\}$ is $\sG$ distinct.
    By the maximal property,
    and the definition of $\sG$ distinctness (see Definition~\ref{defn:G-distinctness}),
    we know that if $p_{\bm j} \in \overline{Q_A}$ with some $\bm j \not \in \mathcal{I}$,
    then there must exist $\fg \in \sG$ and $\bm i \in \mathcal{I}$
    such that $\bm p_{\bm i} = \fg(\bm p_{\bm j})$.

    Given $\varepsilon>0$,
    by the point matching property (Condition 4) we can find $\bm f$ such that
    \begin{itemize}
        \item For $\bm p_{\bm i}$ in general position,
        $|\bm f(\bm p_{\bm i}) - \bm y_{\bm i}| \le \varepsilon$ for all $\bm p_{\bm i} \in \mathcal P$.
        \item For $\bm p_{\bm i}$ not in general position, $|\bm f(\bm p_{\bm i})| \le 1$.
    \end{itemize}
    Then, by the extremeness of $\Pcal$,
    the inequality holds true for all $\bm p_{\bm i}$ such that  $\square_{\bm i,\delta} \subset K$.

    For $\alpha \in (0,1)$, define the shrunken cube
    \begin{equation}
        \square_{\bm i,\delta}^{\alpha}
        :=
        [i_1\delta,(i_1+\alpha)\delta]\times \cdots \times [i_n\delta,(i_n+\alpha)\delta],
    \end{equation}
    and define $K^{\alpha} = \bigcup_{\square_{\bm i,\delta} \subset K} \square_{\bm i,\delta}^{\alpha}$
    to be a subset of $K$.
    Given $\beta > 0$, we now use the coordinate
    zooming property (Condition 3)
    to find $u^{\otimes} \in \Acal$
    such that

    \begin{equation}
        \label{eq:cond-on-u}
        u([ih,(i+\alpha h)]) \subset [ih,(i+\frac{\beta}{n}\delta)]
        \text{ for }i \in \{ i_s : s=1,\dots,n; \bm i \in \Ical\}.
    \end{equation}
    To do this, we first construct a piesewise linear function
    $\tilde u$ such that
        \begin{equation}
            \tilde u|_{[i\delta, (i+\alpha) \delta]}(x) = i+\frac{\beta}{2n} \delta,
        \end{equation}
        by setting
        \begin{equation}
            \tilde u|_{[(i+\alpha) \delta,
            (i+1)\delta]}(x) = (x - (i-\alpha)\delta)/(1 - \alpha)
            + i + \frac{\beta}{2n} \delta
        \end{equation}
        explicitly,
    and select $\varepsilon < \frac{\beta}{3n} \delta$.
    Then we use Coordinate Zooming property with respect to $v = \tilde u$ and $\varepsilon$, to obtain $u$, such that $\| u - \tilde{u}\|_{C(\mathbb I)} \le \varepsilon$, yielding the condition \eqref{eq:cond-on-u} holds.

    Therefore, we have

    \begin{equation}
        \label{eq:esti-p-general-position}
        |\bm{f}(u^{\otimes}(\bm x)) - y_{\bm{i}}|
        \le 2\varepsilon \text{ for }\bm x \in \square_{\bm i,\delta}^{\alpha}
    ,\end{equation}
    where $\bm p_i$ is in general position, and
    \begin{equation}
        \label{eq:esti-p-not-general-position}
        |\bm f(u^{\otimes}(\bm x))| \le 1+ \varepsilon \text{ for }\bm x \in \square_{\bm i,\delta}^{\alpha}
    ,\end{equation}
    where $\bm p_i$ is not in general position.

   These two estimates \eqref{eq:esti-p-general-position} and \eqref{eq:esti-p-not-general-position} will be useful in the final step.

    \paragraph{Step 3.}
    We are ready to estimate the error $\| F - g \circ \bm f \circ u^{\otimes}\|_{L^p(K)}$.
    Since $\| F - g \circ {\bm \varphi}\|_{L^p(K)} \le \varepsilon$,
    it suffices to estimate
    $
        \| g \circ {\bm \varphi} - g \circ  \bm f \circ u^{\otimes} \|_{L^p(K)}
        \le
        \lip(g)
        \|
            {\bm \varphi  } -
            \bm f \circ u^{\otimes}
        \|_{L^p(K)}
    $.

    The estimation is split into three parts,
    \begin{equation} \begin{split}K \setminus K^{\alpha}, \\  K^{\alpha}_1 = \bigcup_{\bm p_{\bm i}\text{ in general position}} \square_{\bm i,\delta}^{\alpha},  \\ K^{\alpha}_2 = \bigcup_{\bm p_{\bm i}\text{ not in general position}} \square_{\bm i,\delta}^{\alpha}.\end{split}
    \end{equation}
    Notice that $K^{\alpha} = \cup \square_{\bm i,\delta}^{\alpha}$.

    For $K^{\alpha}_1$, from \eqref{eq:esti-p-general-position} in the end of Step 2, we have $\|\bm f \circ u^{\otimes} - \bm \varphi_0\|_{L^{\infty}(K_1^{\alpha})} \le 2\varepsilon$, and thus
    \begin{equation}
        \label{eq:K1alpha-estimate}
        \|\bm f \circ u^{\otimes} - {\bm \varphi}_0\|_{L^p(K^{\alpha}_1)}
        \le
        2\varepsilon[\bm \lambda(K^{\alpha})]^{1/p} \le 2\varepsilon[\bm \lambda(K)]^{1/p}.
    \end{equation}

    For $K^{\alpha}_2$, note that if $\bm p_{\bm i}$ does not in general position, then all points in $\square_{\bm i,\delta}$ will be close to a hyperplane $\Gamma_{ij}  := \{\bm x : x_i = x_j \}$ for some distinct $i,j$, the distance from those points to $\Gamma_{ij}$ will be small than $\sqrt{n}\delta$. Therefore, the Lebesgue measure of $K_2^{\alpha} \subset K_2$ will be smaller than that of all points whose distance to the union of hyperplanes $\Gamma_{ij}$ less than $\sqrt{n}\delta$, which is $O(\delta)$. Thus,
    we have
    \begin{equation}
        \label{eq:K2alpha-estimate}
        \begin{split}
        \|\bm f \circ u^{\otimes} - {\bm \varphi}_0\|_{L^p(K^{\alpha}_2)} \le~& (1+ {\varepsilon}+ \|\bm \varphi_0\|_{C(K)})O(\delta) \\ \le ~& (1+ {\varepsilon}+ \|\bm \varphi\|_{C(K)})O(\delta).
        \end{split}
    \end{equation}
    The last line holds since $\|\bm \varphi_0\|_{C(K)} \le \|\bm\varphi\|_{C(K)}$ by construction.

    For $K \setminus K^{\alpha}$, we have
    \begin{equation}
        \label{eq:Kalpha-estimate}
        \begin{split}
        \|\bm f \circ u^{\otimes} - {\bm \varphi}_0\|_{L^p(K \setminus K^{\alpha})} &
        \le (\|\bm f\|_{C(K)}+\|\bm \varphi\|_{C(K)})~\bm \lambda(K \setminus K^{\alpha})^{1/p}\\ & \le (\|\bm f\|_{C(K)}+\|\bm \varphi\|_{C(K)})(1-\alpha^d)^{1/p}[\bm \lambda(K)]^{1/p}.
        \end{split}
    \end{equation}

    We first choose $\delta$ sufficiently small such that the right hand side of \eqref{eq:K2alpha-estimate} is not greater than $\varepsilon$, then choose $\alpha$ such that  $1- \alpha$ is sufficiently small,
    and $(\|\bm f\|_{C(K)}+\|\bm \varphi\|_{C(K)})(1-\alpha^d)^{1/p} \le \varepsilon$.

    Hence, the total error is
    \begin{equation}
        \label{eq:final-estimate}
        \begin{split}
        \|F - g \circ\bm f \circ u^{\otimes} \| \le 3\varepsilon[\bm \lambda(K)]^{1/p} + 2\varepsilon.
        \end{split}
    \end{equation}
    Combining two estimates we conclude the result (with $3\varepsilon[\bm \lambda(K)]^{1/p} + 2\varepsilon$ replacing $\varepsilon$.)
\end{proof}

\subsection{Results on Dynamical Hypothesis Spaces}

Before going to the proofs of the main results,
we first prove the following auxiliary lemma,
which says that the presence of well function
is enough to guarantee the coordinate zooming property, see Condition 3 in Theorem~\ref{thm:framework}. This is also the core part of the previous paper~\cite{li2019deep}.

\begin{lemma}[Well function achieves coordinate zooming property]
    \label{lem:coordinate-zooming}
    For a one-dimensional function $\sigma \in C(\mathbb R)$,
    define its control family under affine invariance as follows:
    \begin{equation}
        \Fcal(\sigma) := \{ x \mapsto w\sigma(ax + b): w,a,b \in \R \}.
    \end{equation}
    Here, $x$ is a one dimensional variable.
    If there exists a one-dimensional well function $\sigma$
    such that $\Fcal(\sigma) \subset \coor(\Fcal)$,
    then $\Acal_{\Fcal}$ has the coordinate zooming property.
\end{lemma}
\begin{proof}
    This follows immediately from the definition of coordinate zooming property and the approximation result in one dimension
    stated below.
\end{proof}

\begin{theorem}[Main result in \cite{li2019deep}, one dimensional case]\label{thm:20191d}
    For $F: \mathbb R \to \mathbb R$ being continuous and increasing, if the one-dimensional control family $\mathcal F$ satisfies
    \begin{enumerate}
        \item For any compact interval $I$
        there exists a Lipschitz $g\in\mathcal{G}$ such that $F(I) \subset g(\R)$.
        \item $\Fcal$ is {affine invariant}.
        \item $\chbar(\Fcal)$ contains a well function.
    \end{enumerate}
    Then for any  compact interval $I \subset \mathbb{R}^n$ and $\varepsilon>0$,
    there exists $\hat{F} \in \mathcal{H}_{ode}(\Fcal, \Gcal)$ such that
    $\| F - \hat{ F}\|_{L^{\infty}(I)} \le \varepsilon.$
\end{theorem}

We define below a notion of partial order that will be used in the proofs.

\begin{definition}[Partial order of points in $\mathbb R^n$]
    For points $\bm x, \bm y \in \R^n$,
    we write $\bm{x} \prec \bm{y}$ to mean $x_i < y_j$ for all $i,j$.
    Note that the symbol $\prec$ only represents a partial order, not a complete order.
    Moreover, we say a point set $X = \{\bm x^1,\bm x^2,\cdots, \bm x^m\}$ is partially ordered,
    if $\bm x^i \prec \bm x^j$ for any $i,j$
    satisfying either
    \begin{itemize}
        \item $i > j$, where $\bm x^i$ and $\bm x^j$ are in general position
        \item $\bm x^i$ is not in general position, but $\bm x^j$ is in general position.
    \end{itemize}
\end{definition}

We now state and prove the main approximation result when the symmetry group is the full
permutation group $\sS$.
Recall the perturbation assumption defined in Definition~\ref{def:pert_prop}.

\begin{theorem}[UAP from Dynamical Systems, $\sS$ version]
    \label{thm:approx-dynamical}
    Suppose that $\Fcal$ is a Lipschitz control family,
    such that $\overline{\Fcal} := \chbar(\Fcal)$ satisfies the perturbation assumption
    (recall Definition~\ref{defn:resolvence}), and
    \begin{enumerate}
        \item There exists a well function $\sigma$,
        such that $\Fcal(\sigma) = \{ a\sigma(bx_1+c): a,b,c \in \mathbb R\} \subset \coor(\overline{\Fcal})$.
        \item There exists a symmetric invariant well function $\tau$ such that $\pm \tau \in \coor(\overline{\Fcal})$.
    \end{enumerate}
    Also, suppose that for any compact $K\subset \R^n$,
    there exists a Lipschitz $g \in \Gcal$ such that $F(K) \subset g(\R^n)$.
    Then, for $1 \le p < \infty$,
    and any $\sS$ invariant mapping $F \in C(\mathbb R^n;\mathbb R)$,
    compact region $K$ and tolerance $\varepsilon > 0 $,
    there exists $\hat{F} \in \Hcalode$ such that
    $\|F - \hat F\|_{L^p(K)} \le \varepsilon$.
\end{theorem}

For convenience, the following proof will assume $\Fcal$
(instead of $\chbar{(\Fcal)}$) satisfies the above conditions.
This turns out to be sufficient, since one can prove
that the control family generated by $\Fcal$
and $\chbar(\Fcal)$ have the same closure
under the topology of compact convergence.
(See Proposition~\ref{prop:A-closure-chbarxx}).

We first state without proof two lemmas that will be used to prove Theorem~\ref{thm:approx-dynamical}.
These two lemmas will be proved later.

\begin{lemma}[Perturbation Lemma]
    \label{lem:perturbation-lemma}
    Suppose that the point set $\{\bm x^1,\cdots, \bm x^m\}$ is $\sG$ distinct
    and $\Fcal$ satisfies the perturbation property
    (Definition~\ref{def:pert_prop}).
    Then, there exists a mapping $\bm \beta\in \Acal_{\Fcal}$
    so that $\bm y^i = \bm \beta (\bm x^i)$, $i=1,\dots,m$,satisfy:
    if $(i,j) \neq (k,l)$ but
    \begin{equation} \label{eq:well-perturbed}
        y_i^j = y_{k}^l,
    \end{equation}
    then both $\bm y^j$ and $\bm y^l$ are not in general position.
    A point set $\{ \bm y^i \}$ satisfying this condition
    will subsequently be called \emph{well-perturbed}.
\end{lemma}

\begin{lemma}[Partially Ordering Lemma]
    \label{lem:partially-ordering}
    Suppose $\Acal = \Acal_{\Fcal}$ is generated by a $\sG$ equivariant Lipschitz control family $\Fcal$,
    and there exists a symmetric invariant well function $\tau(\bm x)$, such that $\pm \tau \in \coor(\Fcal)$.
    Let $A$ be a transversal.
    Consider a finite point set $X = \{\bm x^1,\cdots, \bm x^M\} \subset \overline{Q_A}$
    such that $X$ is well perturbed.
    Then, there exists $\bm \beta \in \Acal$ such that $\bm \beta(X) \subset \overline{Q_A}$ is partially ordered.
    Moreover, we may require that
    \begin{enumerate}
        \item $\bm \beta(X)$ is well perturbed,
        as defined in Lemma~\ref{lem:perturbation-lemma}.
        \item If $\bm x \in X$ is in general position, so is $\bm \beta(\bm x)$.
    \end{enumerate}
\end{lemma}

Assuming these lemmas, we now give a proof of Theorem~\ref{thm:approx-dynamical}.

\begin{proof}[Proof of Theorem~\ref{thm:approx-dynamical}]

We only need to show that the point matching property (see Condition 3 of Theorem~\ref{thm:framework}) holds for such $\Fcal$.
Without loss of generality, we assume that
\begin{equation}
    X = \{ \bm x^1,\bm x^2,\cdots ,\bm x^M\},
\end{equation}
where $\bm x^i~(i = 1,\cdots, m)$ are in general position,
and $\bm x^i~ (i = m+1,\cdots, M)$ are not in general position.
By the perturbation lemma (Lemma~\ref{lem:perturbation-lemma}),
we can find $\bm \alpha \in \Acal_{\Fcal}$ such that $\bm \alpha(X)$ is well perturbed.
Thus applying the partially ordering lemma (Lemma~\ref{lem:partially-ordering}),
we can find $\bm \beta \in \Acal_{\Fcal}$ such that
$\widetilde X = \bm \beta(\bm \alpha(X)) \subset Q$ is partially ordered
and meets the following two requirements as stated in Lemma~\ref{lem:partially-ordering}.
In the rest of the proof, we will use $\widetilde X$ to replace $X$.

We first consider an ideal case to illustrate our idea to the proof:
if all the points in $X$ are in general position (namely, $m = M$),
the proof will be straightforward.
We can assume that the destination point set
\begin{equation}
Y = \{ \bm y^1,\bm y^2,\cdots ,\bm y^m\},
\end{equation}
where $\bm y^i (i = 1,\cdots, m)$ are in general position.
Again, by the construction of Lemma~\ref{lem:partially-ordering},
we can find $\bm \gamma \in \Acal_{\Fcal}$
such that $\widetilde Y = \bm \gamma(Y) \subset Q$ is partially ordered.
We may set $u^{\otimes}$ as a coordinate zooming function such that
\begin{equation}
    |u(\tilde x_j^i) - \tilde y_j^i|
    \le
    \varepsilon / (\lip \bm \gamma^{-1}
    \cdot \lip (\bm \beta \circ \bm \alpha)).
\end{equation}
Therefore, the mapping $\bm \rho = \bm \gamma^{-1} \circ u^{\otimes} \circ \bm \beta \circ \bm \alpha$
can make sure that $|\bm \rho(\bm x^i) - \bm y^i| \le \varepsilon.$

Now we consider the general case,
which needs a slight modification on the destination point set $Y$:
We consider adding a point into $Y$.
Define
\begin{equation}
    Y_+ = \{ \bm y^1,\cdots, \bm y^m, \bm y^{m+1} = \bm z : = \mathbf 0\},
\end{equation}
where $\bm y^i (i = 1,\cdots,m)$ is in general position,
and no coordinate value of them are 0.
This $Y_+$, by definition, is well-perturbed.
Therefore, by the partially ordering lemma (Lemma~\ref{lem:partially-ordering}),
we can find $\bm \gamma \in \Acal_{\Fcal}$ such that
$\widetilde Y_+ = \bm \gamma(\bm \alpha(Y_+)) \subset \overline{Q}$ is partially ordered.
In this case, we may also set $u^{\otimes}$ is a coordinate zooming function such that
$|u(\tilde x_j^i) - \tilde y_j^i|\le  \varepsilon / (\lip \bm \gamma^{-1} \cdot \lip (\bm \beta \circ \bm \alpha))$.
However, the issue we encounter here is that the value of
$\tilde y_j^i$ have not been defined for $i = m+1,\cdots, M$.

To resolve this, we now determine the value of $\tilde y_j^i$ for $i = m+1,\cdots, M$.
Denote by $\tilde{\bm z}:= \bm \gamma (\bm z) = z\mathbf{1}$.
Due to the $\sS$ equivariance, we require
\begin{enumerate}
\item  $|y_j^i - \tilde z| \le \varepsilon / (\lip \bm \gamma^{-1} \cdot \lip (\bm \beta \circ \bm \alpha))$
\item $\tilde y_j^i \ge \tilde y_{j'}^{i'}$ if and only if $\tilde x_{j}^i \ge \tilde x_{j'}^{i'}$.
\end{enumerate}
Therefore, the mapping $\bm \rho = \bm \gamma^{-1} \circ u^{\otimes} \circ \bm \beta \circ \bm \alpha$
ensures that $|\bm \rho(\bm x^i) - \bm y^i| \le \varepsilon$ for $i = 1,2,\cdots, m$,
and $|\bm \rho(\bm x^i)| \le \varepsilon$ for $i = m+1,\cdots, M$.

\end{proof}

For the general case, we prove the following result for any
transitive subgroup $\sG$.
Recall the definition of resolving a group from Definition~\ref{defn:resolvence}.
We hereafter take a fixed transversal $A$ of $\sG$.
\begin{theorem}[UAP from Dynamical Systems, General Version]
    \label{thm:approx-dynamical-general}
    Suppose that $\overline{\Fcal}:= \chbar(\Fcal)$ is a
    Lipschitz control family resolving $\sG$, and
    \begin{enumerate}
        \item There exists a well function $\sigma$, such that $\Fcal(\sigma) = \{
        a\sigma(bx_1+c): a,b,c \in \mathbb R\} \subset \coor(\overline{\Fcal})$.
        \item There exists a symmetric invariant well function $\tau$ such that $\pm
        \tau \in \coor(\overline{\Fcal})$.
    \end{enumerate}
    Also, suppose that for any compact $K\subset \R^n$,
there exists a Lipschitz $g \in \Gcal$ such that $F(K) \subset g(\R^n)$.
Then, for $1 \le p < \infty$,
and any $\sG$ invariant mapping $F \in C(\mathbb R^n;\mathbb R)$,
compact region $K$ and tolerance $\varepsilon > 0 $,
there exists $\hat{F} \in \Hcalode$ such that
$\|F - \hat F\|_{L^p(K)} \le \varepsilon$.
\end{theorem}

As discussed before, it is sufficient
in the following proof we will assume conditions 2 and 3
hold for $\Fcal$ instead of $\chbar{(\Fcal)}$.

The proof is based on the Lemmas~\ref{lem:perturbation-lemma} and~\ref{lem:partially-ordering} above, and two additional lemmas below.
\begin{lemma}
    \label{lem:one-point}
    Suppose $\Fcal$ satisfies the conditions in
    Theorem~\ref{thm:approx-dynamical-general}.
    Then for each $\bm x \in Q_A$,
    there exists $\bm \gamma \in \Acal_{\Fcal}$ such that $\bm \gamma (\bm x) \in Q$.
\end{lemma}

\begin{lemma}
    \label{lem:plugging}
    Suppose that $\{\bm x^1,\cdots, \bm x^m\}$ are partially ordered,
    and that $\Fcal$ satisfies the conditions in Theorem~\ref{thm:approx-dynamical-general}.
    If for some $\bm x^s$, there exists ${\bm \zeta}^{\circ} \in \Acal_{\Fcal}$ such that
    $\bm \zeta^{\circ}(\bm x^s) \in Q$
    Then we can find ${\bm \zeta} \in \Acal_{\Fcal}$ such that
    \begin{enumerate}
    \item $\bm \zeta(\bm x^1),\cdots, \bm \zeta(\bm x^m)$ are partially ordered.
    \item For $i \neq s$, if $\bm x^i$ is in the cross section $Q_{\fa}$,
    then so is $\bm \zeta(\bm x^i)$.
    \item $\bm \zeta(\bm x^s) \in Q$.
    \end{enumerate}
\end{lemma}

\begin{proof}[Proof of Theorem~\ref{thm:approx-dynamical-general}]
    By the same techniques in the proof of Theorem~\ref{thm:approx-dynamical},
    we can assume that there exists $\bm \alpha \in \Acal_{\Fcal}$,
    such that $\bm \alpha(X) \subset \overline{Q_A}$ is partially ordered,
    and satisfies the requirements in Lemma~\ref{lem:partially-ordering}.
    Next, we assert that, there exists $\bm \beta \in \Acal_{\Fcal}$ such that
    $\bm \beta$ can send all points in general position into $Q$, that is,
    \begin{enumerate}
        \item For $\bm x$ in general position, $\bm \beta \circ \bm \alpha(\bm x) \in Q$.
        \item $\bm \beta \circ \bm \alpha(X)$ is partially ordered.
    \end{enumerate}
    The remaining proof is identical to the proof of Theorem~\ref{thm:approx-dynamical}.
    With a little abuse of notations, we use $\bm x^i$ to denote $\bm \beta \circ \bm \alpha (\bm x^i)$. Now, the conditions in Lemma~\ref{lem:partially-ordering} read,
    \begin{enumerate}
        \item For $\bm x$ in general position, $\bm x \in Q$.
        \item $X$ is partially ordered.
    \end{enumerate}
    Now let us deal with the assertion,
    by Lemma~\ref{lem:cross-section} we can find ${\bm \zeta}_1^{\circ} \in \Acal_{\Fcal}$
    such that ${\bm \zeta}_1^{\circ} (\bm x^1) \in Q$.
    By Lemma~\ref{lem:plugging},
    we can modify this $\widehat{\bm \zeta}_1$ to ${\bm \zeta}_1$,
    such that
    \begin{enumerate}
        \item $\bm \zeta_1(\bm x^1),\cdots, \bm \zeta_1(\bm x^m)$ is partially ordered.
        \item For $i \neq s$, if $\bm x^i$ is in the same cross section $Q_{\fa}$, then so is $\bm \zeta_1(\bm x^i)$.
        \item $\bm \zeta_1(\bm x^1) \in Q$.
    \end{enumerate}
    Sequentially,
    we can find $\bm \zeta_2,\cdots, \bm \zeta_m$ to map all points in general position in $Q$,
    therefore
    \begin{equation}
        \bm \beta = \bm \zeta_m \circ \bm \zeta_{m-1} \circ \cdots \circ \bm \zeta_1
    \end{equation}
    satisfies the assertion.
\end{proof}

The rest of this subsection aims to prove lemmas~\ref{lem:perturbation-lemma},
\ref{lem:partially-ordering}, \ref{lem:one-point} and \ref{lem:plugging}.

\subsubsection{Proof of Lemma~\ref{lem:perturbation-lemma}: Perturbation Lemma}

We recall the definition of the perturbation property from Definition~\ref{defn:resolvence}.
Define the \emph{similarity} of two points,
at least one of which is in general position,
as
\begin{equation}
    \label{eq:similarity2}
    \overline{s}(\bm x, \bm y) =
    {
        | \{ (i,i'):x_i = y_{i'}\} | .
    }
\end{equation}
When both points are not in the general position, the similarity is defined as zero.
Assuming the perturbation property is satisfied,
$\overline{s}(\bm x, \bm y) \neq 0$ implies
the existence of $\bm f \in \Fcal$ and a coordinate zooming function $u^{\otimes}$ such that
\begin{equation}[\bm f(u^{\otimes}(\bm x))]_i \neq [\bm f(u^{\otimes}(\bm y))]_{i'}\end{equation}
holds for some $i$ such that $x_i = y_{i'}$.

\begin{proof}[Proof of Lemma~\ref{lem:perturbation-lemma}]
    We first extend the similarity $\overline{s}$ in \eqref{eq:similarity2} to a set, namely, define
    \begin{equation}
        \overline{s}(Y) =
        \sum_{j \neq j'} \overline{s}(\bm y^j,\bm y^{j'}).
    \end{equation}
    The goal of perturbation lemma is to show there
    exists $\bm \beta \in \Acal_{\Fcal}$ such that
    $\overline{s}(\bm \beta(X)) = 0$.
    Consider $\bm \beta \in \Acal_{\Fcal}$ which minimizes  $\overline{s}(\bm \beta(X))$,
    and for convenience we set $Y = \bm \beta(X)$,
    which means $\bm y^j = \bm \beta (\bm x^j)$.
    We assert that the quantity $\overline{s}(Y)$ should be zero.

    Otherwise, there exists $J,J'$ such that $\overline{s}(\bm y^J, \bm y^J)>0$.
    By definition of the perturbation property,
    we can find $\bm f \in \Fcal$ and $u^{\otimes}$ such that
    at for some $I$ and $I'$,
    $[\bm f(u^{\otimes}(\bm y^{J}))]_I \neq [\bm f(u^{\otimes}(\bm y^{J'}))]_{I'}$
    but $y_I^J = y_{I'}^{J'}$.

    Then, consider $\bm \gamma_{t} = \bm \phi(\bm f, t) \circ u^{\otimes}$.
    For sufficiently small $t$,
    observe that
    no new  $(i,i',j,j')$ will meet the requirement
    \begin{equation}
        [\bm \gamma_t(\bm y^j)]_i =  [\bm\gamma_t(\bm y^{j'})]_{i'},
        \qquad
        \text{and}
        \qquad
        y_i^j \neq y_{i'}^{j'}.
    \end{equation}
    Therefore, the quantity
    \begin{equation}
        \mathcal{S}(t) := \sum_{j\neq j'} \overline{s}(\bm \gamma_t(\bm y^j), \bm \gamma_t(\bm y^{j'})
    \end{equation}
    will be non-increasing for sufficiently small $t$.
    That is, for sufficiently small $t > 0$, it holds that $\mathcal{S}(t) \le \mathcal{S}(0)$.
    Notice that the argument holds for general $\bm f$ and $u$.
    Now, we use the perturbation property to conclude the proof by showing
    $\mathcal{S}(t) < \mathcal{S}(0)$ for sufficiently small $t > 0$.
    Since the flow map $\bm \gamma_t$ satisfies for some $y_i^J = y_{i'}^{J'}$,
    but
    \begin{equation}
        [\bm \gamma_t(\bm y^J)]_i \neq [\gamma_t(\bm y^{J'})]_{i'}.
    \end{equation}
    Therefore, we can ensure that for sufficient small $t$,
    \begin{equation}
        \overline{s}(\bm \gamma_t(\bm y^J), \bm \gamma_t(\bm y^{J'}))
        < \overline{s}(\bm y^J, \bm y^{J'}),
    \end{equation}
    Thus, it contradicts the minimal choice of $\bm \beta$,
    since $\overline{s}(\bm \gamma_t \circ \bm \beta(X)) < \overline{s}(\bm \beta(X))$.
    This immediately yields that the minimal value of $\overline{s}(\bm \beta(X))$ is 0.


\end{proof}

\subsubsection{Proof of Lemma~\ref{lem:partially-ordering}: Partially Ordering Lemma}

In this proof, two kind of mappings will be used.
The first kind is the symmetric invariant well function
$\tau$ with zero interval $\I$ such that $\tau \in \coor(\Fcal)$.
Since $\tau$ is $\sS$ invariant, by the definition of the $\coor$ operator,
the tensor-product mapping $\bm \tau \in \Fcal$ such that
$\bm \tau(\bm x) = \tau^{\otimes}(\bm x) = [\tau(\bm x),\tau(\bm x),\cdots,\tau(\bm x)]$
is in $\Fcal$. Recall that Lemma~\ref{lem:partially-ordering} has two additional requirements:
\begin{enumerate}
    \item $\bm \beta(X)$ is well perturbed,
    as defined in Lemma~\ref{lem:perturbation-lemma}.
    \item If $\bm x \in X$ is in general position, so is $\bm \beta(\bm x)$.
\end{enumerate}
Consider the following dynamical system for $\bm x$:
\begin{equation}
    \label{eq:hdw-flow}
    \frac{d}{dt}\bm x = \bm \tau(\bm x).
\end{equation}
Clearly, the flow map $\bm \phi(\bm \tau,T)$ of $\bm \tau$
at any time horizon $T$ satsfies
$[\bm \phi(\bm \tau,T)\bm x ]_i - [\bm \phi(\bm \tau,T)\bm x]_j = x_i - x_j$.
Hence, the second requirement is always satisfied
if we use \eqref{eq:hdw-flow} in our construction.

The other kind is the collection of all coordinate zooming functions.
By Lemma~\ref{lem:coordinate-zooming}, these are contained in $\overline{\Acal_{\Fcal}}$.
These preserve the order of all coordinates and
thus the second requirement is also satisfied.
Since we will only use these two kinds of mappings
in our subsequent constructions,
it suffices to prove that a composition of these
two types of functions can be constructed to
satisfy the well perturbed property.

\begin{proof}[Proof of Lemma~\ref{lem:partially-ordering}]
~
    \paragraph{Step 1.}
    Consider $\mathcal B$, defined as the family of mappings
    consisting of all composition of $\bm \phi(\bm \tau, t)$
    and coordinate zooming functions.
    \begin{equation}
        \begin{split}
        \mathcal B :=
        \{
            &\bm f_1 \circ \bm f_2\circ \cdots \circ \bm f_p : \\
            &\text{ each } \bm f_i \in \Acal \text{ is either }
            \bm \phi(\bm \tau, t_i) \text{ for some } t_i
            \text{ or a coordinate zooming function}.
        \}
        \end{split}
    \end{equation}
    Obviously, all mappings in $\Bcal$ will map each $Q_{\fa}$ into itself.
    As a consequence, if $\bm x$ is not in general position, then $\bm \alpha(\bm x)$ is not either.
    Thus, we may assign an order in $\{m\}$ such that
    $i \gg j$ if and only if
    \begin{itemize}
       \item  $i >  j$, if $\bm x^i$ and $\bm x^j$ are in general position, or
       \item $\bm x^i$ is not in general position, while $\bm x^j$ is in general position.
    \end{itemize}
    By Proposition~\ref{prop:composition-closure}, it holds that $\mathcal B \subset \overline{\Acal_{\Fcal}}$.
    We first show that there exists $\bm \beta \in \mathcal B$  that
    $[\bm\beta(\bm x^i)]_1 > [\bm \beta(\bm x^j)]_1$ if $i \gg j$.
    For convenience, $i \ll j $ means $j \gg i$.
    By the definition of well-perturbedness, $i \ll j$ means that $x_1^i \neq x_1^j$.

    Now, let us consider
    \begin{equation}
         b_{\min}= \min |
         \big\{
             (i,j) : i \ll j, [\bm \beta(\bm x^i)]_1 < [\bm \beta(\bm x^j)]_1,
             \bm \beta \in \mathcal B \text{ and } \bm \beta(X) \text{ is well perturbed}.
        \big\} |
    \end{equation}
    and $\bm \beta$ achieves this minimum.
    It suffices to show $ b_{\min} = 0$.
    Suppose not, take $I$ and $J$ such that $I  \ll J$ and $[\bm \beta(\bm x^I)]_1 < [\bm \beta(\bm x^J)]_1$,
    and $I$ and $J$ is taken to minimize
    $[\bm \beta(\bm x^j)]_1 - [\bm \beta(\bm x^i)]_1$ under the aforementioned conditions.
    We show that there are no other $[\bm \beta(\bm x^k)]_1$ between them.
    In fact, if $[\bm \beta(\bm x^I)]_1 < [\bm \beta(\bm x^k)]_1 < [\bm \beta(\bm x^J)]_1$,
    then since either $I<k$ or $k<J$ is satisfied, contradicting with the minimal choice of $I$ and $J$.

    We now give a direct construction to show that $\bm \beta$ is not minimal.
    With a little abuse of notation, We also use $X$ to denote this new $\bm \beta (X)$.
    Since $X$ is in a general position now, we have either
    $\min(x^J_k) < \min(x^I_k)$ or $\min(x^J_k) > \min(x^I_k)$.

    Our construction is based on discussing them separately.
    By the definition of symmetric invariant well function,
    we can find $a>1$ such that if $|x^I| > a$ and $x^j \in \I$ then
    $|\tau(\bm x)| > \delta$. Since $\tau$ is continuous,
    we may assume that $|\tau(\bm x)| \le \Delta $
    for $|\bm x| \le 2a$ for some $\Delta$.
    \begin{itemize}
        \item If $x^J_K := \min_k(x^J_k) < \min_k(x^I_k)$.
        We first choose a coordinate zooming function $u^{\otimes} \in \Acal$ such that
        \begin{enumerate}
            \item
        \begin{equation}\label{eq:lemma-cond1}|u(x^J_1) - u(x^I_1)| \le \varepsilon^2 \text{ and }|u(x^k_1) - u(x^l_1)| \ge \varepsilon\end{equation}
            for sufficiently small $\varepsilon<1$. The detailed value of $\varepsilon$ will be determined later.
            \item \begin{equation}\label{eq:lemma-cond2}|u(x^j_i)| < 2a-1, \forall i,j, \text{ and }u(x^J_i) > -a - 1, \forall j.\end{equation}
            \item \begin{equation}\label{eq:lemma-cond3} (u(x^I_j) - \varepsilon, u(x^I_j) + \varepsilon) \subset \I, \forall j.\end{equation}
        \end{enumerate}
        With a slight abuse of notation,
        we use $\bm x^i$ to denote $u^{\otimes}(\bm x^i)$ for $i = 1,2,\cdots,M$.
        Then, the flow $\bm \phi_{\bm \tau}$ and condition \eqref{eq:lemma-cond3}
        ensure that $\bm \phi(\pm \bm \tau,t )(\bm x^I) = \bm x^I$
        by the definition of $\I$.
        Next, we choose $\hat{\bm \tau}= \tau$ or $-\tau$ such that
        $[\bm \phi({\hat{\bm \tau}},\bm x^J,t)]_1$ is decreasing in $t$.
        By the construction of $u$ in Eq.~\eqref{eq:lemma-cond2},
        when $t < \frac{\varepsilon}{\Delta}$,
        we have $[\bm \phi({\hat{\bm \tau}},t)\bm x^J]_1 - x^J_1 \le \delta t$.

        Hence, we achieve our goal if $\delta t > \varepsilon^2$.
        Set $t \in (2\varepsilon^2\delta^{-1}, 3\varepsilon^2\delta^{-1})$ to be determined.
        We have $t \in \frac{\varepsilon}{\Delta}$ for $\varepsilon < \delta / (4\Delta)$.
        Since the set of $t$'s such that $\bm \phi({\hat{\bm \tau}}, t)X$
        is in a general position is open and contains $t = 0$,
        such a $t$ must be found, by condition \eqref{eq:lemma-cond1}.
        \item If $\min_k(x^J_k) > \min_k(x^I_k) =: x^I_K$.
        The arguments are similar except we exchange the role of $I$ and $J$, and change decreasing into increasing.
    \end{itemize}
    In either scenario, we deduce that $\bm \beta$ is not minimal, thus $\bm \beta(X)$ must be partially ordered.

    \paragraph{Step 2.}
    Substituting $X$ by $\bm \beta(X)$, we now assume that
    $x_1^i > x_1^j$ if $i \ll j$.
    We now show that there exists $\bm \gamma \in \mathcal B$ such that $\bm \gamma(X)$ is partially ordered.
    Suppose there exists
    $\bm \alpha \in \mathcal B$ such that
    \begin{equation}
        \bm \alpha(\bm x^1) \succ \bm \alpha(\bm x^2) \succ \cdots \bm \alpha(\bm x^I) \succ \bm \alpha(\bm x^k)
    \end{equation}
    for $k = I+1,\cdots, M$, while
    \begin{equation}
        [\bm \alpha(\bm x^1)]_1> [\bm \alpha(\bm x^2)]_1 > \cdots [\bm \alpha(\bm x^m)]_1.
    \end{equation}
    We give a direct construction of $\bm \alpha'$
    uch that the above inequality holds for $I+1$, provided $\bm x^{I+1}$ is in general position.
    Using this argument recurrently will yield the desired result.

    Consider the case $I = 1$, and we also use $\bm x^i$ to denote $\bm \alpha(\bm x^i)$.
    Consider a coordinate zooming function such that only $u(x_1^1)>a$ is outside the interval $\I$.
    Then, the flow map $\bm \gamma_t = \bm \phi({\bm \tau}, t) \circ u^{\otimes}$
    will only move $\bm x^1$ while keeping other $\bm x^i$ fixed.
    At some time horizon $T < \infty$,
    we have $\bm \gamma_t(\bm x^1) \succ \bm \gamma_t(\bm x^i)$ for $i \neq 1$.
    Otherwise, we will obtain both $\bm \gamma_t(\bm x^1)$ is unbounded and
    $\min_{i}[\bm \gamma_t(\bm x^1)]_i \le x_1^2$,
    which is a contradiction since $[\bm \gamma_t(\bm x^1)]_i- [\bm \gamma_t(\bm x^1)]_{i'}$ is a constant.

    For the other $I$, we apply the same technique
    to find $\bm \zeta \in \Acal_{\Fcal}$ such that for
    $\bm \beta = \bm \zeta \circ \bm \alpha$,
    we have $\bm \beta(\bm x^{I}) \succ \bm \beta(\bm x^{k})$ for $k > I$.
    The only difference is that we need
    to modify $\bm \zeta$ such that
    it preserves the order of $\bm x^1$ to $\bm x^{I-1}$.

    To do this, we need to move
    $\bm x^1,\dots,\bm x^{I-1}$ to an appropriate
    location so that modifications to the
    remaining coordinates do not induce
    a change of ordering for these coordinates.
    More precisely,
    suppose that $\bm x^{I-1}$ is in general position,
    and moreover we assume that $\bm x^{I-1} \in Q_{\fa}$.
    Since $\bm \zeta$ must be a bijection, implying that we can choose
    $\bm y^{I-1} \in Q_{\fa}$ such that $\bm \zeta(\bm y^{I-1}) \succ \bm \beta(\bm x^{I-1})$.
    Then, we choose a coordinate zooming function $u^{\otimes}$ to send
    $\bm \alpha(\bm x^{I-1})$ to $\bm y^{I-1}$.
    Let $\tilde{\bm \beta} = \bm \zeta \circ u^{\otimes} \circ \bm \alpha$.
    Then, the function $\tilde{\bm \beta}$ satisfies that
    $\tilde{\bm \beta}(\bm x^{I-1}) \succ \tilde{\bm \beta}(\bm x^{I})$.

    For the general case, using the same technique,
    we can sequentially determine the destination of $\bm x^{I-1},\cdots, \bm x^1$,
    denoted as $\bm y^{I-1}, \cdots \bm y^1$, to fulfill the requirement.

\end{proof}

To close this section, we give the proofs of Lemmas~\ref{lem:one-point}
and~\ref{lem:plugging} that was used to prove
Theorem~\ref{thm:approx-dynamical-general}.

\subsubsection{Proof of Lemma~\ref{lem:one-point}}

\begin{proof}[Proof of Lemma~\ref{lem:one-point}]
It suffices to show that if $\fa$ and $\fb$ is $\Fcal$ directly connected,
then for given point $\bm x \in Q_{\fa}$,
there exists $\bm \alpha \in \Acal_{\Fcal}$
such that $\bm \alpha (\bm x) \in Q_{\fb}$.
By direct connectivity there exists
$i,j$ such that $\bm x \in Q_{\fa}$ we have $x_i > x_j$,
and for $\bm x \in Q_{\fb}$ we have $x_i < x_j$.

Suppose $\bm z \in \partial Q_{\fa} \cap \partial Q_{\fb}$
and $\bm f \in \Fcal$ satisfies the assumption on direct connectivity,
that is,
\begin{equation}
[\bm f(\bm z)]_i \neq [\bm f(\bm z)]_j.
\end{equation}
By the continuity of $\bm f$,
we also have $[\bm f(\bm y)]_i \neq [\bm f(\bm y)]_j$
for $\bm y$ in some neighborhood $U$ of $\bm z$.

Let $\varepsilon < 1$ be a small constant whose value will be determined later,
such that the ball $B(\bm z,3\varepsilon) \in U$.
Choose $\bm y \in Q_{\fa}$ such that $y_i - y_j < \varepsilon^2/2$
while for other $(k,l) \neq (i,j)$,
we require that $|y_k - y_l| > \varepsilon - \varepsilon^2/2$.

Consider a coordinate zooming function that $|u(x_k) - y_k| \le \varepsilon^2/2$. Hence we have $u^{\otimes}(\bm x) \in B(\bm z,2\varepsilon)$, and
\begin{enumerate}
    \item $|u(x_i) - u(x_j)|\le \varepsilon^2$,
    \item For other pairs $(k,l) \neq (i,j)$ , we have $|u(x_k) - u(x_l)|\ge \varepsilon.$
\end{enumerate}

For $\bm y \in B(\bm z, 2\varepsilon)$
we can choose $a>0$ such that $[\bm f(\bm y)]_i > [\bm f(\bm y)]_j + a$,
and $|\bm f(\bm y)| \le \frac{1}{a}$,
e.g. take $a < \min(\frac{1}{2}([\bm f(\bm y)]_i - [\bm f(\bm y)]_j), (|\bm f(\bm y)| + 1)^{-1})$.
Then, consider the following dynamics
\begin{equation}
    \frac{d}{dt}\bm y = -\bm f(\bm y).
\end{equation}
Simple estimation yields that the flow $\bm \phi({- \bm f},t)(u^{\otimes}(\bm x)) \in Q_{\fb}$ for $t \in (
\frac{\varepsilon^2}{a}, \frac{a}{4}\varepsilon) $. We choose $\varepsilon$ sufficiently small such that the interval is non-empty, and choose the mapping as $\bm \phi({-\bm f}, t)\circ \bm u^{\otimes} \in \Acal_{\Fcal}$.

\end{proof}

\subsubsection{Proof of Lemma~\ref{lem:plugging}}
\begin{proof}[Proof of Lemma~\ref{lem:plugging}]

We consider the simplest case, when $\bm \zeta^{\circ} := \bm \phi(\bm f,T)$ is the flow map of some $\bm f \in \Fcal$. The general case can be obtained by repeating the following argument. Denote by

\begin{equation}
\delta := \min \{|x^i_j - x^{i'}_{j'}|:x^i_j \neq x^{i'}_{j'}.\}
\end{equation}
and there exists $t_0$ such that
\begin{equation}
\sup_{0 \le t \le t_0} \lip [\bm \phi({\bm f}, t) - id] \le \frac{\delta}{4}.
\end{equation}
Hence, consider $\bm y = \bm \phi(\bm f,t_0)\bm x$, then we have
\begin{equation}
    |y^i_j - y^{i'}_{j'}| \ge \frac{\delta}{4},
\end{equation}
if $y_i^j \neq y_{i'}^{j'}$.

Therefore, $\bm y^1,\cdots, \bm y^m$ is partially ordered,
and we have if $\bm x^i \in Q_{\fa}$ then $\bm y^i \in Q_{\fa}$
since the order of each coordinate does not change.
Now consider the coordinate zooming function $u$, such that
\begin{equation}
|u(y^i_j) - x^i_j| \le \frac{\delta}{8}
\end{equation}
for $i \neq s$, and
\begin{equation}
|u(y^s_j) - y^s_j| \le \varepsilon
\end{equation}
for some small $\varepsilon$ whose value will be determined later.

Therefore, we can apply $\bm \phi({\bm f}, t_0)$ for $u^{\otimes}(\bm y^1), u^{\otimes}(\bm y^2),\cdots, u^{\otimes}(\bm y^m)$, and get the similar result. This iteration will break down after $[T/t_0]+1$ iterations, and we denote by $\bm \zeta$ the final composited iteration. It suffices to control the error on $\bm x^s$, by a telescope expansion we have
\begin{equation}
|[\bm \zeta(\bm x^s)]_j - x^s_j| \le \varepsilon L([T/t_0]+1),
\end{equation}
where $L = \sup_{0 \le t \le t_0} \lip \bm \phi({\bm f}, t) \le 1+ \frac{\delta}{4}$. Choose $\varepsilon < 1$ such that the right hand side is not greater than $\frac{1}{4} \min_{i\neq j}|x^s_i - x^s_j|$, which guarantees that $\bm \zeta(\bm x^s)$ and $\bm x^s$ are in the same $Q_{\fa}$.
\end{proof}

\subsection{Verification of the Proposed Architectures}
\label{sec:proofs:verification}
In this section, we provide a direct verification that
the proposed architectures in Section~\ref{sec:application} satisfy
the assumptions of our universal approximation theorems, thereby
guarantee their UAP for their respective symmetry groups $\sG$.
Clearly, Condition 1 and 2 of Theorems~\ref{thm:approx-dynamical} and
\ref{thm:approx-dynamical-general} are all satisfied as
a consequence of Remark~\ref{rmk:symmetrc-invariant-well-function}.
It suffices to check the perturbation assumption and
$\sG$ transversal transitivity on each case,
to show that $\Fcal$ resolves $\sG$ (see Definition~\ref{defn:resolvence}).
For convenience, we recall the definition of each $\Fcal$ and $\sG$,
and $\sigma$ denotes a well function.
The following facts that hold for any well function $\sigma$
will be frequently used:

\begin{enumerate}
    \item
    For given $x$ and $y \in \mathbb R$,
    we can find $b \in \mathbb R$ such that $\sigma(x+b) = 0$ but $\sigma(y+b)\neq 0$.
    \item If $x_{1},\cdots, x_{n}$ and $y_1,\cdots,y_n$ satisfy that
    \begin{equation}
        \sum_{i=1}^n \sigma(px_i+q) = \sum_{i=1}^n \sigma(py_i+q)
    \end{equation}
    for all $p,q \in \mathbb R$, then it holds that $\{ x_i:i = 1,2,\cdots, n\} = \{y_i: i = 1,2,\cdots n\}.$ Hereafter, the sets are understood as multi-sets.

\end{enumerate}

\subsubsection{Shift Invariant Cases}
We show that each $\Fcal_{\ast, i}$ resolves $\sT$, where
\begin{equation}
    \begin{aligned}
        \Fcal_{\ast, 1} :=& \{ v \sigma (\bm w \ast \bm  x + b \bm 1)
        : \bm w \in \R^{d_1\times \cdots d_n}, v,b \in \R\}, \\
        \Fcal_{\ast,2} :=& \{ \bm w \ast \sigma (v\bm x + b \bm 1), \bm w \in \R^{d_1\times \cdots d_n}, v,b \in \R\}.
    \end{aligned}
\end{equation}

\begin{proof}[Proof of Corollary~\ref{cor:shift-invar-small}]
For the perturbation assumption, without loss of generality,
we assume that $x_1 = y_1$.
It suffices to notice that for two point $\bm x$ and $\bm y$
such that they are $\sT$ distinct, there exists $i$ such that $x_i \neq y_i$.
We choose $\bm w = \bm e_i$, whose value is $1$ in the $i$ coordinate,
and $0$ in other coordinates.
Then, we choose an appropriate $b$
such that $\sigma(x_i + b) = 0$ while
$\sigma(y_i + b) \neq 0$,
therefore
\begin{equation}
    [\bm w \ast \sigma(\bm x + b \mathbf 1)]_1 \neq
    [\bm w \ast \sigma(\bm y + b \mathbf 1)]_1
\end{equation}
and
\begin{equation}
    [ \sigma(\bm w \ast\bm x + b \mathbf 1)]_1 \neq
    [\sigma(\bm w \ast \bm y + b \mathbf 1)]_1.
\end{equation}
Hence we show both $\Fcal_{\ast,i}$ satisfy the perturbation assumption.

For transversal transitivity,
we show that for each $\fa, \fb \in \sS$, $\fa$
and $\fb$ are $\Fcal$ directly connected.
For simplicity, we consider the case when
\begin{equation}
    Q_{\fa} = \{ x_1 > x_2 > \cdots > x_n \}
\end{equation}
and \begin{equation}
Q_{\fb} = \{x_2 > x_1 > \cdots > x_n \},
\end{equation}
and the other cases are similar.
In this case, we can choose $\bm z \in \partial Q_{\fa} \cap \partial Q_{\fb}$
such that $z_{2} \neq z_3$,
We choose $\bm w = \bm e_2$, whose value is $1$ in second coordinate,
and $0$ in other coordinates.
Then, by a similar approach as what we did in the proof of Lemma~\ref{lem:perturbation-lemma},
we can prove that $\fa$ and $\fb$ are $\Fcal$ directly connected.

\end{proof}
\subsubsection{Full Permutation Invariant Cases}
We show that each $\Fcal_{s,i}$ ($i=1,2$) resolves $\sS$, where

\begin{align}
    \Fcal_{s,1} :=&
    \{
        a \sigma(w\bm x + v \Sigma\bm x + b\bm 1)
        :
        a, w, v,b \in \R
    \},\\
    \Fcal_{s,2} :=&
    \{
        a\sigma(w\bm x + c \bm 1) + b\Sigma(\sigma(v\bm x) + d \bm 1)
        :
        a, b, c,d,w, v \in \R
    \}.
\end{align}

\begin{proof}[Proof of Corollary~\ref{cor:symmetric}]
    \label{proof:cor:symmetric}
Since any transversal only contains one element,
it suffices to show $\Fcal_{s,i}$ satisfies the perturbation assumption.
For $\Fcal_{s,1}$, if $\bm x$ and $\bm y$ are $\sS$ distinct,
then there exists a function $u^{\otimes} \in \Fcal$ such that
\begin{equation}
    u(x_1) + u(x_2) \cdots + u(x_d) \neq u(y_1) + u(y_2) \cdots + u(y_d).
\end{equation}
For $\Fcal_{s,1}$, we set $a = 1, w = 0$,
and choose a suitable $b$ as before such that
\begin{equation}
    \sigma(\bm \Sigma u^{\otimes}(\bm x) + b \mathbf 1)
    \neq \sigma(\bm \Sigma u^{\otimes}(\bm y) + b \mathbf 1).
\end{equation}
For $\Fcal_{s,2}$, we set $a =w =  0,b = 1$, and the rest is similar to $\Fcal_{s,1}$.

\end{proof}

Additionally, we prove the following result for Janossy pooling (with respect to a Lipschitz function $\phi$) for completeness.
\begin{proposition}
The following control family
\begin{equation}
\label{eq:janossy-1}
\Fcal_{s,Ja,1} = \{ v\sigma(a\bm x + b\sum_{i=1}^n \phi(cx_i+d) + c\bm 1)\}
\end{equation}
contains a symmetric well function and satisfies the perturbation property provided that there exists a neighborhood $U$ of $\phi(0)$, such that $\phi^{-1}(U)$ is bounded.
\end{proposition}
\begin{proof}
We first prove that the control family \eqref{eq:janossy-1} contains a symmetric invariant well function, by showing the following function
$f = \sigma(k[\frac{1}{n}\sum_{i=1}^n\phi(x_i) - \phi(0)])$ for a suitable $k$ is a symmetric invariant well function. Without loss of generality, we assume that $\sigma$ is a well function such that $\sigma(z) = 0$ for $|z| < r$ and $\sigma(z) \neq 0$ for $|z| > R$. The assumption on $\phi$ now implies that there exists $M$ and $\delta$ such that if $|z|>M$ then $|\phi(z) - \phi(0)|> \delta$.
Thus, we take $k = R/\delta$, so that $f$ is a symmetric well function with $a = r/(k\lip \phi)$ and $m = M$.

We now prove the perturbation property. It suffices to show that if $\bm x$ and $\bm y$ are $\sS$ distinct, then there exists a function $u^{\otimes}$ such that
\begin{equation}
    \phi(u(x_1))+\phi(u(x_2))+\cdots+\phi(u(x_n)) \neq \phi(u(y_1))+\phi(u(y_2))+\cdots+\phi(u(y_n)).
\end{equation}
We prove it by contradiction. Suppose that
\begin{equation}
    \phi(u(x_1))+\phi(u(x_2))+\cdots+\phi(u(x_n)) =  \phi(u(y_1))+\phi(u(y_2))+\cdots+\phi(u(y_n))
\end{equation}
holds for all coordinate zooming functions but $x_1 > y_1$.
Then, we can carefully choose $u$ such that $u(x_2),\cdots,u(x_n)$ and $u(y_1),\cdots,u(y_n)  \in B(u(0),\varepsilon)$ for some small $\varepsilon$, such that $\phi(u(x_1)) \in U$.
However, since the value of $u(x_1)$ can be arbitrarily chosen, the assumption yields the contradiction.
As a result, the control family \eqref{eq:janossy-1} satisfies the perturbation property.
\end{proof}

\subsubsection{Product Permutation Invariant Cases}
We prove that
\begin{equation}
    \Fcal_{s,2D,1}=\{ v \sigma ( w_0\bm x + w_{r,1}\Sigma_{r,1}\bm x + w_{c,1} \Sigma_{c,1}\bm x  + c \bm 1): w_0, w_{r,1},w_{c,1},c \in \mathbb R \}
\end{equation}
resolves $\sS_{2D}$, and the higher dimensional cases are similar.

\begin{proof}[Proof of Corollary~\ref{cor:product-permutation}]
    First, we check the perturbation property.
    As before, it suffices to prove that there exists $(i_1,j_1)$
    and $(i_2,j_2)$, real numbers $c_1,c_2$, and an increasing function $u$, such that $x_{i_1,j_1} = y_{i_2,j_2}$, but
    \begin{equation}
    c_1 \sum_{i' = 1}^{d_1} u(x_{i',j_1}) +c_2 \sum_{j' = 1}^{d_2}u(x_{i_1,j'}) \neq c_1 \sum_{i' = 1}^{d_1} u(y_{i',j_2}) +c_2 \sum_{j' = 1}^{d_2}u(y_{i_2,j'}).
    \end{equation}
    Suppose on the contrary that for all $(i_1,j_1)$ and $(i_2,j_2)$
    such that $x_{i_1,j_1} = y_{i_2,j_2}$, then
    \begin{equation}
        \label{eq:c1-c2}c_1 \sum_{i' = 1}^{d_1} u(x_{i',j_1}) +c_2 \sum_{j' = 1}^{d_2}u(x_{i_1,j'}) = c_1 \sum_{i' = 1}^{d_1} u(y_{i',j_2}) +c_2 \sum_{j' = 1}^{d_2}u(y_{i_2,j'})
    \end{equation}
    holds for all choice of $c_1,c_2$, and increasing function $u$.
    It follows from \eqref{eq:c1-c2} that
    \begin{equation}
        \sum_{i' = 1}^{d_1} u(x_{i',j_1}) =  \sum_{i' = 1}^{d_1} u(y_{i',j_2})
    \end{equation}
    and
    \begin{equation}
        \sum_{j' = 1}^{d_2}u(x_{i_1,j'}) = \sum_{j' = 1}^{d_2}u(y_{i_2,j'}).
    \end{equation}
    After a permutation, we may assume that
    $x_{1,1} = y_{1,1}$ and therefore $x_{i,1} = y_{i,1}$ and $x_{1,j} = y_{1,j}$ for all $i,j$.

    Suppose now, $x_{I,J} \neq y_{I,J}$, then there must exist $i \neq I$ such
    that $x_{I,J} = y_{i,J}$ and $j \neq J$ such that $x_{I,J} = y_{I,j}$.
    Therefore, $\bm y$ is not in general position, similarly we can deduce that
    $\bm x $ is not in general position, contradicting the assumption
    $\overline{s}(\bm x, \bm y)>0$.

    For transversal transitivity, it suffices to consider for example
    \begin{equation}Q_{\fa} = \{ x_{1,1} > x_{1,2} \cdots, > x_{d_1,d_2} \} \end{equation}
    and
    \begin{equation}Q_{\fb} =  \{ x_{1,2} > x_{1,1} \cdots, > x_{d_1,d_2} \}.\end{equation}
    Clearly, we can choose $\bm z \in \partial Q_{\fa} \cap \partial Q_{\fb}$ such that $\sum_{j=1}^{d_2} z_{1,j} \neq \sum_{j=1}^{d_2} z_{2,j}$,
    implying transversal transitivity.

We now check the perturbation property.
As before, it suffices to prove that there exists $(i_1,j_1)$
and $(i_2,j_2)$, real numbers $c_1,c_2$, and an increasing function $u$, such that $x_{i_1,j_1} = y_{i_2,j_2}$,  but
\begin{equation}
c_1 \sum_{i' = 1}^{d_1} u(x_{i',j_1}) +c_2 \sum_{j' = 1}^{d_2}u(x_{i_1,j'}) \neq c_1 \sum_{i' = 1}^{d_1} u(y_{i',j_2}) +c_2 \sum_{j' = 1}^{d_2}u(y_{i_2,j'}).
\end{equation}
Suppose that, for all $(i_1,j_1)$ and $(i_2,j_2)$ such that $x_{i_1,j_1} = y_{i_2,j_2}$, we have
\begin{equation}
    \label{eq:c1-c2}c_1 \sum_{i' = 1}^{d_1} u(x_{i',j_1}) +c_2 \sum_{j' = 1}^{d_2}u(x_{i_1,j'}) = c_1 \sum_{i' = 1}^{d_1} u(y_{i',j_2}) +c_2 \sum_{j' = 1}^{d_2}u(y_{i_2,j'}).
\end{equation}
holds for all choice of $c_1,c_2$, and increasing functions $u$.
It follows from \eqref{eq:c1-c2} that
\begin{equation}
    \sum_{i' = 1}^{d_1} u(x_{i',j_1}) =  \sum_{i' = 1}^{d_1} u(y_{i',j_2})
\end{equation}
and
\begin{equation}
    \sum_{j' = 1}^{d_2}u(x_{i_1,j'}) = \sum_{j' = 1}^{d_2}u(y_{i_2,j'}).
\end{equation}
Since $u$ is an arbitrary increasing function, we deduce that the equalities
\begin{equation}\{x_{i,j_1}:i=1,2,\cdots,d_1\} = \{y_{i,j_2}:i=1,2,\cdots,d_1\}\end{equation}
and
\begin{equation}\{x_{i_1,j}:j=1,2,\cdots,d_2\} = \{y_{i_2,j}:j=1,2,\cdots,d_2\}.\end{equation}
After a permutation, we may assume that
$x_{1,1} = y_{1,1}$ and therefore $x_{i,1} = y_{i,1}$ and
$x_{1,j} = y_{1,j}$ for all $i,j$.

Suppose now, $x_{I,J} \neq y_{I,J}$, then there must exist $i \neq I$ such that $x_{I,J} = y_{i,J}$ and $j \neq J$ such that $x_{I,J} = y_{I,j}$. Therefore, $\bm y$ is not in general position.
Similarly, we can deduce that $\bm x $ is not in general position,
contradicting the assumption $\overline{s}(\bm x, \bm y)>0$.

For transversal transitivity, it suffices to show for example that
\begin{equation}Q_{\fa} = \{ x_{1,1} > x_{1,2} \cdots, > x_{d_1,d_2} \} \end{equation}
and
\begin{equation}Q_{\fb} =  \{ x_{1,2} > x_{1,1} \cdots, > x_{d_1,d_2} \}. \end{equation}
are $\Fcal_{s,2D,1}$-connected.
Clearly, we can choose $\bm z \in \partial Q_{\fa} \cap \partial Q_{\fb}$ such
that $\sum_{j=1}^{d_2} z_{1,j} \neq \sum_{j=1}^{d_2} z_{2,j}$. Therefore, we can
show the transversal transitivity.
\end{proof}

\section*{Acknowledgements}

We are grateful for discussions with Isaac P. S. Tian
and Tonio Buonassisi on the applications of the theory developed to materials modelling.

This research was supported by the National Research Foundation, Singapore
under the NRF fellowship (project No. NRF-NRFF13-2021-0005).
T.L. is partially supported by The Elite Program of
Computational and Applied Mathematics for PhD Candidates in Peking University.
Z.S is supported under the Distinguished Professorship of National University of Singapore.

\newpage

\appendix
\section{Basic Properties of Equivariant Mappings}

In these appendices, we collect a number of auxiliary technical results
required in the proofs of the main results.


\subsection{Closure Properties of Invariant and Equivariant Mappings}
\begin{proposition}
    If $\mathcal F$ is $\sG$ equivariant, then so is $\mathcal A_{\Fcal}$.
\end{proposition}
\begin{proof}
    It suffices to show that the flow map of a equivariant mapping is equivariant,
    since function composition preserves equivariance.
    Consider the ODE system $d\bm x / dt = \bm f(\bm x)$.
    The Picard iteration sequence will converge to the solution of ODE,
    since $\bm f$ is Lipschitz.
    Clearly each function on Picard iteration is equivariant,
    then it follows from the fact that $\bm x$
    is Lipschitz that the flow map is also equivariant.
    Precisely, given $\bm x_0(t) = \bm 0$, define the Picard iteration as follows:
    \begin{equation}
        \bm x_{n}(t) = \bm x_0 + \int_0^{t} \bm f(\bm x_{n-1}(s)) ds
    \end{equation}
    for $n = 1,2,\cdots$. Define the flow maps for the Picard iterates
    \begin{equation}
        \bm \phi_n(\bm f, t)
        :=
        \bm x_0 \mapsto \bm x_n(t).
    \end{equation}
    By induction, each $\bm \phi_n(\bm f, t)$ is $\sG$ equivariant, and $\bm \phi_n(\bm f, t) \to \bm \phi(\bm f, t)$ uniformly in any compact set $K \subset \mathbb R^n$, therefore by Proposition~\ref{prop:negative-equivar}, the mapping $\bm \phi(\bm f, t)$ is $\sG$ equivariant.

\end{proof}

Using $\coor(\Fcal)$ we can make the connection between well function and coordinate zooming function. See Lemma~\ref{lem:coordinate-zooming} below.


\

\subsection{Convex Hull Property of Equivariant Mapping}
\label{sec:equi:chbar}

In this subsection, we show some basic property of convex hull closure.
These propositions imply that condition 2 in
Theorem~\ref{thm:main} can be reduced to requiring $\coor(\Fcal)$ to
contains a symmetric invariant well function.
The argument goes as follows:
by Proposition \ref{prop:chbar-coor-commutation},
the conditions on Theorem~\ref{thm:approx-dynamical} holds for $\chbar(\Fcal)$.
Whence, by Theorem~\ref{thm:approx-dynamical} and Proposition~\ref{thm:level-set},
we can deduce that the assertion in Theorem~\ref{thm:main}
holds for $\chbar(\Fcal)$.
Applying Proposition~\ref{prop:A-closure-chbarxx},
we obtain that the assertion also holds for $\Fcal$.

\begin{proposition}
    \label{prop:chbar-coor-commutation}
Given a group $\sG$, the $\coor$ operator commutes with $\chbar$ operator, that is,
\begin{equation}
    \coor(\chbar(\Fcal)) = \chbar(\coor(\Fcal)).
\end{equation}
\end{proposition}
\begin{proof}
    Notice that by Proposition~\ref{prop:negative-equivar}
    we know that $\chbar(\Fcal)$ is equivariant.
The commutation between $\coor$ and $\ch$ is obvious. We only prove that
\begin{equation}
    \label{eq:commutation-bar-coor}
    \coor(\overline{\Fcal}) = \overline{\coor(\Fcal)}.
\end{equation}
Suppose $f_1 \in \coor(\overline{\Fcal})$, which means
there exists $\bm f \in \overline{\Fcal}$ such that $\coor(\bm f) = f_1$.
For any compact set $K$ and tolerance $\varepsilon$,
by the definition of $\overline{\Fcal}$,
there exists $\bm g \in \Fcal$ such that $\|\bm f - \bm g\|_{C(K)} \le \varepsilon$.
Hence we have $\|f_1 - \coor(\bm g)\|_{C(K)} \le \varepsilon$,
yielding that $f_1 \in \overline{\coor(\Fcal)}$.

Conversely, if $f_1 \in \overline{\coor(\Fcal)}$, then for any compact set $K$
(assumed to be invariant)
and tolerance $\varepsilon>0$, we can find $g_1 \in \coor(\Fcal)$,
then by Proposition~\ref{prop:representation} we can construct
$\bm f$ and $\bm g \in \Fcal$, such that $\|\bm f - \bm g\|_{C(K)} \le \varepsilon$.
Therefore $f_1 \in \coor(\overline{\Fcal})$.
Hence the equality in \eqref{eq:commutation-bar-coor} holds.

\end{proof}

\begin{proposition}\label{prop:composition-closure}
If $\bm f, \bm g\in \overline{\Acal_{\Fcal}}$, then so does $\bm f\circ \bm g $.
\end{proposition}
\begin{proof}
    This follows immediate from the Lipschitz property of $\Fcal$.
\end{proof}

\begin{proposition}
    \label{prop:A-closure-chbarxx}
For any $\hat{\bm \varphi} \in \Acal_{\chbar(\Fcal)}$, compact set $K \in \mathbb{R}^n$ and tolerance $\varepsilon>0$, we can find a mapping $\bm \varphi \in \Acal_{\Fcal}$, such that $\|\bm \varphi - \hat{\bm \varphi}\|_{L^p(K)} \le \varepsilon.$
\end{proposition}
\begin{proof}
The proof of the one-dimensional is found in \cite{li2019deep},
and the proof of higher-dimensional cases are analogous, hence omitted.
\end{proof}

\vskip 0.2in
\endappendix

\bibliographystyle{plain}
\bibliography{library}

\end{document}